\documentclass[11pt]{article}
\usepackage[body={7in, 9in},left=1in,right=1in]{geometry}

\usepackage{natbib}
\usepackage{wrapfig}
\usepackage{graphicx}
\usepackage{amsmath}
\usepackage{amsthm}
\usepackage{amssymb}
\usepackage{multirow}
\usepackage{xspace}
\usepackage{algorithm}
\usepackage{algorithmic}
\usepackage[pagebackref=true,colorlinks=false]{hyperref}

\usepackage{sidecap}

\usepackage{tikz}
\usetikzlibrary{petri,shapes,scopes,arrows,decorations.pathmorphing,decorations,backgrounds,positioning,fit,calc,matrix}
\usepackage{subfigure}

\newtheorem{thm}{Theorem}
\newtheorem{defn}{Definition}

\newtheorem{lem}{Lemma}

\newcommand{\lpart}{y}
\newcommand{\Lpart}{Y}

\newcommand{\Lspace}{\mathcal{Y}}

\newcommand{\tphi}{\tilde{\phi}}

\newcommand{\E}[2]{\mathbb{E}_{#1}\left[#2\right]}
\newcommand{\bbE}{\mathbb{E}}

\newcommand{\mcal}[1]{\mathcal{#1}}

\newcommand{\argmax}{\operatornamewithlimits{argmax}}
\newcommand{\argmin}{\operatornamewithlimits{argmin}}
\newcommand{\ceil}[1]{\left \lceil #1 \right \rceil}
\newcommand{\ang}[1]{\left \langle #1 \right \rangle}
\newcommand{\cY}{\mcal{Y}}
\newcommand{\cX}{\mcal{X}}

\newcommand{\cC}{\mcal{C}}
\newcommand{\ncliques}{|\cC|}
\newcommand{\cS}{\mcal{S}}

\newcommand{\cL}{\mathcal{L}}
\newcommand{\cA}{\mathcal{A}}
\newcommand{\hG}{\hat{G}}
\newcommand{\hR}{\hat{R}}
\newcommand{\out}[1]{}
\newcommand{\reals}{\mathbb{R}}

\newcommand{\ind}[1]{\mathbf{1}\left[#1\right]}

\newcommand{\bh}{\mathbf{h}}
\newcommand{\bv}{\mathbf{v}}
\newcommand{\bz}{\mathbf{z}}

\newcommand{\hinge}[1]{\left[#1\right]_+}

\newcommand{\mmarg}[1]{\theta^\star(x,#1)}   
\newcommand{\score}[1]{\theta(x,#1)}         

\newcommand{\mmargit}[1]{\theta^{t\star}(x^i,#1)}   

\newcommand{\pmmarg}[1]{\theta_p^\star(x,#1)}   
\newcommand{\pscore}[1]{\theta_p(x,#1)}         
\newcommand{\pscoremax}[0]{\theta_p^\star(x)}   

\newcommand{\pscorei}[1]{\theta_p(x^i,#1)}         

\newcommand{\ssp}[0]{\theta_x}              
\newcommand{\sspRV}[0]{\theta_X}              
\newcommand{\wit}[1]{y^\star(x,#1;\theta)}          
\newcommand{\pwit}[1]{y^\star(x,#1;\theta_{p})}          

\newcommand{\pthr}{\tau(x;\theta_{p},\alpha)}
\newcommand{\thr}[0]{\tau(x;\theta,\alpha)}

\newcommand{\pthri}{\tau(x^{i};\theta_{p},\alpha)}
\newcommand{\thri}{\tau(x^{i};\theta,\alpha)}

\newcommand{\suchthat}{\quad \textrm{ s.t. } \quad}

\newcommand{\SPC}{\ensuremath{\textrm{SPC}}\xspace}

\newcommand{\bw}{\mathbf{w}}
\newcommand{\bft}{\mathbf{f}}

\newcommand{\todo}[1]{\textcolor{red}{{\bf TODO:} #1}}

\makeatletter
\DeclareRobustCommand\onedot{\futurelet\@let@token\@onedot}
\def\@onedot{\ifx\@let@token.\else.\null\fi\xspace}

\newcommand{\st}{\quad\textrm{s.t.}\quad}                        

\begin{document}

\title{Structured Prediction Cascades}

\author{David Weiss, Benjamin Sapp, Ben Taskar \\
       {\tt djweiss@cis.upenn.edu}, {\tt bensapp@cis.upenn.edu}, {\tt taskar@cis.upenn.edu}\\
       \\Department of Computer and Information Science\\
       University of Pennsylvania\\
       Philadelphia, PA 19104-6309, USA
     }

\maketitle

\begin{abstract}
  Structured prediction tasks pose a fundamental trade-off between the need for model complexity to increase predictive power and the limited computational resources for inference in the exponentially-sized output spaces such models require. We formulate and develop the {\em Structured Prediction Cascade} architecture: a sequence of increasingly complex models that progressively filter the space of possible outputs. The key principle of our approach is that each model in the cascade is optimized to accurately filter and refine the structured output state space of the {\em next} model, speeding up both learning and inference in the next layer of the cascade. We learn cascades by optimizing a novel convex loss function that controls the trade-off between the filtering efficiency and the accuracy of the cascade, and provide generalization bounds for both accuracy and efficiency.  We also extend our approach to intractable models using tree-decomposition ensembles, and provide algorithms and theory for this setting. We evaluate our approach on several large-scale problems, achieving state-of-the-art performance in handwriting recognition and human pose recognition. We find that structured prediction cascades allow tremendous speedups and the use of previously intractable features and models in both settings.
\end{abstract}

\section{Introduction}

The classical trade-off between approximation and estimation error
(bias/variance) is fundamental in machine learning. In regression and
classification problems, the {\em approximation error} can be reduced
by increasing the complexity of the model at the cost of higher {\em
  estimation error}. Standard statistical model selection techniques 
 \citep{mallows73,vapnik1974theory, aic74,devroye1996probabilistic, barron1999risk,bartlett02selection} explore a hierarchy of models of increasing complexity primarily to minimize expected error,
 without much concern for the
computational cost of using the model at test time.

However, in structured prediction tasks, 
 such as
machine translation, speech recognition, articulated human pose
estimation and many other complex prediction problems,  {\em test-time computational constraints} play a critical role as
models with increasing inference complexity are considered. In
these tasks, there is an exponential number of possible predictions
for every input. Breaking these joint predictions up into independent
decisions (e.g., translate each word independently, recognize a
phoneme at a time, detect arms separately) ignores critical
correlations and leads to poor performance. On the other hand,
structured models used for these tasks, such as grammars and graphical
models, can capture strong dependencies but at considerable cost of
inference. For example, a first order conditional random field
(CRF) \citep{lafferty01crf} is fast to evaluate but may not be an
accurate model for phoneme recognition, while a fifth order model is
more accurate, but prohibitively expensive for both learning and
prediction. Model complexity can of course also lead to over-fitting problems 
due to the sparseness of the training data, but this
aspect of the error is fairly well understood and controlled using
standard regularization and feature selection methods.

In practice, model complexity is limited by computational constraints
at prediction time, either explicitly by the user or implicitly
because of the limits of available computation power. We therefore
need to balance {\em inference error} with {\em inference
  efficiency}. A common solution is to use heuristic pruning
techniques or approximate search methods in order to make more complex
models feasible. For example, in statistical machine translation,
syntactic models are combined with n-gram language models to produce
impractically large inference problems, which are heavily and
heuristically pruned in order to fit into memory and any reasonable
time budget \citep{hiero,venugopal2007efficient,
  petrov-haghighi-klein:2008:EMNLP}. However,
previous work remains unsatisfactory in several respects: (1) model
parameters are not learned specifically to balance the
accuracy/efficiency trade-off, but instead using remotely related criteria,
and (2) no optimality or generalization guarantees exist.

In this paper, we address the accuracy/efficiency trade-off for
structured problems by learning a {\em cascade} of structured
prediction models, in which the input is passed through a sequence of
models of increasing computational complexity before a final
prediction is produced. The key principle of our approach is that each
model in the cascade is optimized to accurately filter and refine the
structured output state space of the {\em next} model, speeding up
both learning and inference in the next layer of the cascade. Although
complexity of inference increases (perhaps exponentially) from one
layer to the next, the {\em state-space sparsity} of inference increases
exponentially as well, and the entire cascade procedure remains highly
efficient. We call our approach {\em Structured Prediction Cascades}
(\SPC).

The contributions of this paper are organized as follows.\footnote{Preliminary analysis and applications of structured prediction cascade was developed in
  \citep{weiss10,sapp10cascades,weiss10ensemble}.}
\begin{itemize}
\item In Section \ref{sec:sc-tree}, we describe the \SPC inference
  framework for tree-structured problems where sparse exact inference
  is tractable.  We also propose a tree-decomposition method for applying cascades
  to loopy graphical models in Section \ref{sec:sc-loopy}.
\item In Section \ref{sec:learning}, we describe how cascades can be
  learned to achieve a desired accuracy/efficiency trade-off on
  training data. We introduce a novel convex loss function
  specifically geared for learning to filter accurately and
  effectively, and describe a simple stochastic subgradient algorithm for
  learning a cascade one layer at a time.
\item In Section \ref{sec:theory}, we provide a theoretical analysis
  of the accuracy/efficiency trade-off of the cascade 
  We develop novel generalization bounds for both \emph{accuracy and efficiency} of a
  structured prediction model.
\item In Section \ref{sec:applications}, we explore in depth two
  applications of the \SPC framework in which the cascades achieve
  best-known performance. In Section \ref{sec:sequences}, we show how
  \SPC can be applied to linear-chain models for handwriting
  recognition.
  In Section \ref{sec:pose}, we demonstrate the use of
  \SPC for single-frame human pose estimation using a pictorial structures tree
  model cascade. Finally, in Section \ref{sec:videopose}, we show how
  \SPC can be applied to the estimating pose in
  video, using the framework for loopy graphical models introduced in section \ref{sec:sc-loopy}.
\end{itemize}



\section{Related Work}

The trade-off between computation time and model complexity, which is
the central focus of this work, has  been studied before in
several settings we outline in this section.

\paragraph{Training-time Computation Trade-Offs.} Several recent works
have considered the trade-off between estimation error and computation (number of examples processed) at
\textit{training} time in large-scale classification using
stochastic/online optimization methods, notably
\citep{shalevshwartz2008soi,bottou-bousquet-2008}. We use such
stochastic sub-gradient methods in our learning procedure (Section
\ref{sec:learning}). 
In more recent theoretical work, \citet{agarwal11} also address the
issue of estimation time by incorporating computational constraints
into the classical empirical risk minimization framework. However, as
above, \citet{agarwal11} assume that model selection requires choosing
between different methods with fixed test-time computational cost for all
examples. In this paper, we instead analyze adaptive computational
trade-offs in structured inference at test-time, and analyze the trade-offs in
terms of novel loss functions measuring efficiency and accuracy. 

\out{
In the structured prediction setting
considered in this work, stochastic estimation methods require
inference as part of learning. If inference is infeasible, as is the
case for sufficiently complex structured models, then approximate
inference methods such as LP relaxations
\citep{komodakis2007dualdecomp} can be used during learning, but these
do not explicitly address the trade-off and may still be
insufficient. An alternative approach is to replace the exponentially many 
constraints in the estimation problem with far fewer constraints
\citep{sontag10:psuedomax}, but this is often insufficient as it is
guaranteed to succeed only in the limit of infinite data. In contrast
to these previous approaches, our method addresses structured prediction problems
and learns to address the
tradeoff between error and computation directly, allowing for
learning and inference even with very high order models.
}

\paragraph{Test-time Computation Trade-Offs.} The issue of controlling computation at test-time also comes up in
kernelized classifiers, where prediction speed depends on the number
of ``support vectors''.  Several algorithms,
including the Forgetron and Randomized Budget Perceptron \citep{Crammer03onlineclassification,dekel2008forgetron,cavallanti2007tbh}, are designed
to maintain a limited active set of support vectors in an online
fashion while minimizing error. However, unlike our approach, these
algorithms learn a model that has a fixed running time for each test
example. In contrast, our approach addresses structured prediction problems and 
has a {\em example-adaptive} computational
cost that allows for more computation time on more difficult examples,
and greater efficiency gains on examples where simpler models suffice.

\paragraph{Cascades/Coarse-to-fine reasoning.} For binary classification, cascades of classifiers have been quite
successful for reducing computation.  \citet{geman2001} propose a
coarse-to-fine sequence of binary tests to detect the presence and
pose of objects in an image.  The learned sequence of tests is trained
to minimize expected computational cost.  The extremely popular
classifier of \citet{viola02} implements a cascade of
boosting ensembles, with earlier stages using fewer features to
quickly reject large portions of the state space. More recent work on
binary classification cascades has focused on further increasing
efficiency, e.g. through joint optimization \citep{Lefakis10} or
selecting features at test time \citep{GaoK11}.  Our cascade framework
is inspired by these binary classification cascades, but poses new
objectives, inference, and learning algorithms, to deal with the
structured inference setting.

In natural language parsing, several works
\citep{charniak2000maximum,carreras2008tag,petrov:PhD} use a
coarse-to-fine idea closely related to ours and \citet{geman2001}: the
marginals of a simple context free grammar or dependency model are
used to prune the parse chart for a more complex grammar. We compare
to this idea in our experiments.  The key difference with our work is
that we explicitly learn a sequence of models tuned specifically to
filter the space accurately and effectively.  Unlike the work of
\citet{petrov:PhD}, however, we do not learn the structure of the
hierarchy of models but assume it is given by the designer.  \citet{rush12} apply the ideas developed in our preliminary
work to the problem of dependency parsing in natural language
processing. \citet{rush12} learn a cascade of simplified parsing
models using the objective presented in section \ref{sec:learning} to
achieve state-of-the-art performance in dependency parsing
across several languages at about two orders of magnitude less time. 

\citet{pff-cascade} proposed a cascade
for a structured parts-based object detection model.  Their cascade
works by early stopping while evaluating individual parts, if the
combined part scores are less than fixed thresholds.  While the form
of this cascade can be posed in our more general framework (a cascade
of models with an increasing number of parts), we differ from
\citet{pff-cascade} in that our pruning is based on thresholds that
adapt based on inference in each test example, and we explicitly learn
parameters in order to prune safely and efficiently. In
\citet{geman2001,viola02,pff-cascade}, the focus is on preserving
established levels of accuracy while increasing speed.

\section{Structured Prediction Cascades (\SPC)}
\label{sec:sc-tree}

Given an input space $\cX$, output space $\cY$, and a training set
$\{(x^1,y^1),\dots,(x^n,y^n)\}$ of $n$ samples from a joint
distribution $D(X,Y)$, the standard supervised learning task is to
learn a hypothesis $h: \cX \mapsto \cY$ that minimizes the expected
loss $\E{D}{ \cL\left(h(X) , Y\right)}$ for some non-negative loss
function $\cL : \cY \times \cY \rightarrow \reals^+$. In {\em
  structured prediction problems}, $Y$ is a $\ell$-vector of variables
and $\cY = \cY_1 \times \dots \times \cY_\ell$, and $\cY_i = \{1,
\dots, K\}$. In many settings, the number of random variables, $\ell$,
differs depending on input $X$, but for simplicity of notation, we
assume a fixed $\ell$ here. Note that for the rest of this paper, we
will use capital letters $X$ and $Y$ to denote random variables drawn
from $D(X,Y)$ and lower-case letters $x$ and $y$ to denote specific
values of $X$ and $Y$.  We use subscripts to index elements of $y$,
where $y_{i}$ is the ith component of $y$.

The linear hypothesis class we consider is 
\begin{equation}
  \label{eq:structured-inference}
  h(x) = \argmax_{y\in\cY} \theta^\top \bft(x,y),
\end{equation}
where the scoring function is the inner product of a
vector of parameters $\theta \in \reals^{d}$ and a feature function $\bft:
\cX\times\cY\mapsto\reals^d$ mapping $(x,y)$ pairs to a set of
features.  We further make the standard assumption that
$\bft$ decomposes over a set of cliques $\cC$ over output variables,
so that 
\begin{equation}
  \label{eq:structured-inference-cliques}
\theta^\top\bft(x,y) = \sum_{c \in \cC} \theta^\top\bft_c(x,y_c) . 
\end{equation}
We use the notation $y_c$ to denote the subset of variables involved in
clique $c$, $y_c \triangleq \{y_i \mid i \in c\}$. Similarly, we use
$\cY_c \triangleq  \cY_{i_{1}}\times \ldots \times \cY_{i_{|c|} }$ where $c = \{i_{1},\ldots,i_{|c|}\}$, to refer to the set of all assignments to
$y_c$. 
By considering different cliques over $X$ and $Y$, $\bft$ can
represent arbitrary interactions between the components of $x$ and
$y$. Computing the $\argmax$ in $h(x)$ is tractable for low-treewidth (hyper)graphs
but is NP-hard in general, and approximate inference is
typically used when graphs are not low-treewidth.
We will abbreviate $\theta^\top \bft(x,y)$ as $\score{y}$ below, and similarly
 $\theta^\top \bft_{c}(x,y)$ as $\score{y_{c}}$.

In this section, we introduce the framework of Structured Prediction
Cascades (\SPC) to handle problems for which the inference problem
in Eq.~\ref{eq:structured-inference} is prohibitively expensive. For
example, in a 5-th order linear chain model for handwriting
recognition or part-of-speech tagging, $K$ is about $50$ characters or 
parts-of-speech, and exact
inference is on the order $50^{6} \approx 15$ billion times the length
the sequence.  In tree-structured models we have used for human
pose estimation \citep{sapp10cascades}, typical $K$ for each part
includes image location and orientation and is on the order of
$250,000$, so computing $K^{2}$ pairwise features is
prohibitive. Rather than learning a single monolithic model, {\em a
  structured prediction cascade} is a coarse-to-fine sequence of
increasingly complex models $\theta^0, \dots, \theta^T$ with
corresponding features $\bft^0, \dots, \bft^T$. For example, inference
complexity scales exponentially with Markov order in sequence models, and 
quadratically with
spatial/angular resolution in pose models. The goal of each model is
to filter out a large subset of possible values for $y$ without
eliminating the correct one, so that the next level only has to
consider a much reduced state-space.  The filtering process is
feed-forward, and each stage runs inference to compute {\em
  max-marginals} which are used to eliminate low-scoring node or
clique assignments.

In summary, a high-level overview of the \SPC inference framework is
as follows.  Below, $\cS^{i}$ denotes a sparse (filtered) version of the output space
$\cY$: 

\begin{itemize}
\item Given an input $x$, initialize the cascade with $\cS^0 = \cY$.
\item Repeat for each level $i = 0, \dots, T-1$ of the cascade:
  \begin{itemize}
  \item Run sparse inference over $\cS^i$ using model $\theta^{i}(x,y')$ and eliminate a subset of
    low-scoring outputs.
  \item Output $\cS^{i+1}$ for the next model.
  \end{itemize}
\item Predict using the final level:  $y = \argmax_{y' \in \cS^T} \theta^{T}(x,y')$.
\end{itemize}
The process is illustrated in Figure \ref{fig:cascade-flowchart}.  See Figure 
\ref{fig:cascade-output-example} for a concrete example of
the output of a the first two stages of a cascade for handwriting
recognition (Figure \ref{fig:cascade-output-example}) and human pose estimation 
(Figure \ref{fig:pose_overview}).
We will discuss how to represent and choose $\cS^{i}$ in the
next section.  The key challenge is that $\cS^{i}$ are exponential in the number
of output variables, which rules out explicit representations.  The representation
we propose is implicit and concise.  It is also tightly integrated with parameter estimation algorithm for $\theta^{i}$ that optimizes the overall accuracy and efficiency of the cascade.
\begin{figure}[t]
  \centering
  \includegraphics[width=.8\textwidth]{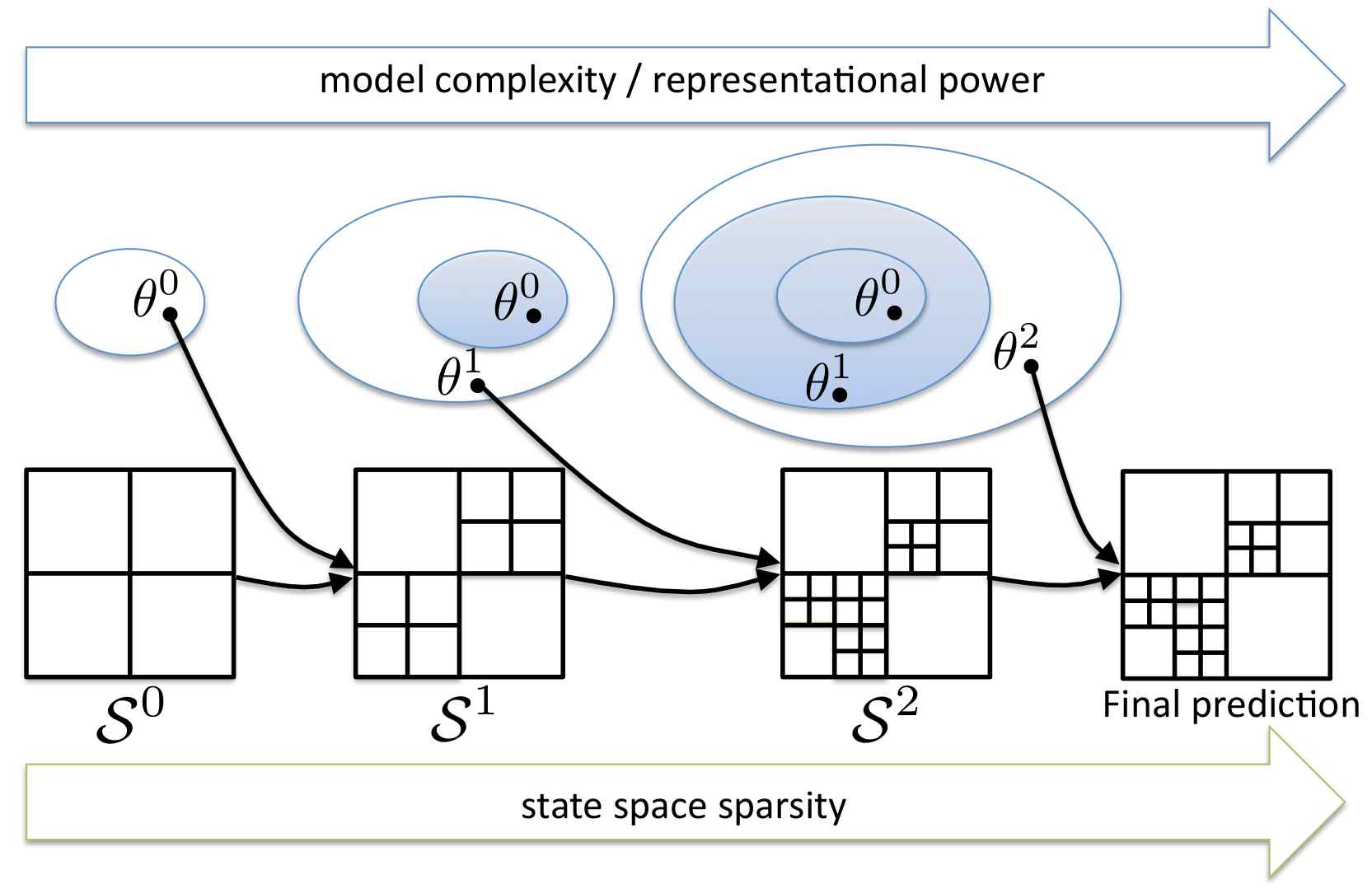}
  \caption{\small A high-level overview of the \SPC inference framework.  As 
the cascade progresses, the representational power of the models increases, yet 
tractability is maintained by sufficient filtering of the state space. }
  \label{fig:cascade-flowchart}
\end{figure}

\begin{table}
  \centering
  \begin{tabular}{|cl|}
\hline
    Symbol & Meaning \\
\hline
    $\cX, X, x$  & input space, variable and value \\
    $\cY, Y, y$  & output space, variables and value \\
    $\cC, c, y_{c}$ & set of cliques, individual clique, clique assignment\\
    $\bft(x,y)$  & features of input/output pair\\
    $\bft_{c}(x,y_{c})$ & features of a clique assignment\\
    $\score{y} \triangleq \theta^{\top}\bft(x,y)$ & score of input/output pair\\
    $\score{y_{c}} \triangleq \theta^{\top}\bft_{c}(x,y_{c}) $ & score of a clique assignment\\
    $\mmarg{y_{c}} \triangleq \max_{y':y'_{c}=y_{c}}\score{y'}$ & max-marginal of a clique assignment $y_{c}$\\
  $\wit{y_c} \triangleq \argmax_{y': y'_c = y_c}\score{y'}$ & best scoring output consistent with clique assignment $y_{c}$\\
    \hline
  \end{tabular}
  \caption{Summary of key notation.}
  \label{tab:notation}
\end{table}


\begin{figure}[t]
  \centering
  \includegraphics[width=.7\textwidth]{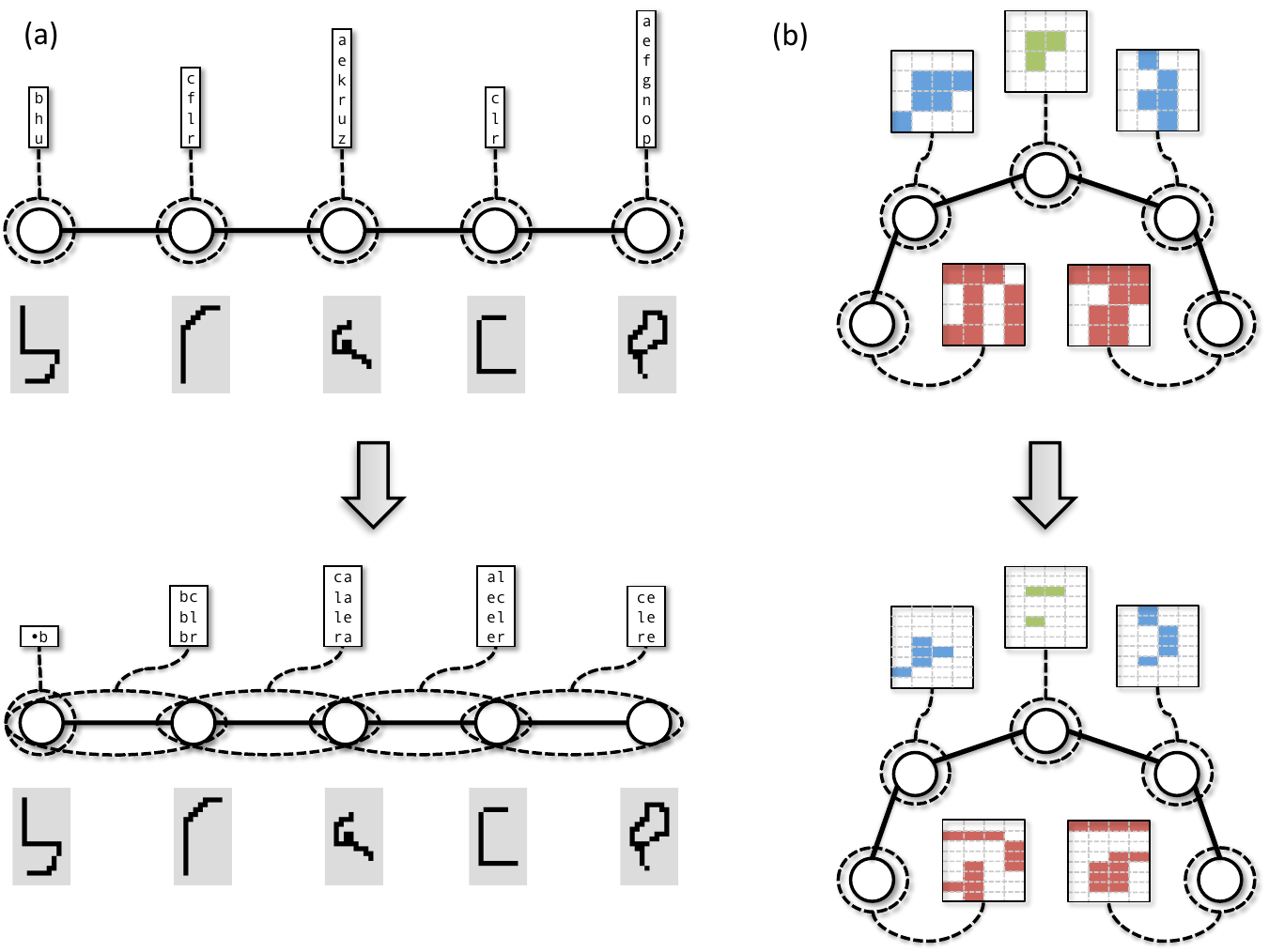}
  \caption{\small Sample output from the first two layers of a
    cascade. Circles represent output variables, and the dashed lines
    indicate cliques that are being filtered at a given level of the
    cascade, with the attached tables representing the sparse state
    space. The solid lines indicate the graph used for inference and
    features. {\bf (a)} 
    Output from a handwriting recognition cascade
    (Section \ref{sec:sequences}) of increasing Markov order. The first
    level outputs a sparse set of possible letters for each
    image. The second level takes as input the sparse set of letters,
    and further refines this to a very sparse set of {\em bigrams} at
    each position. {\bf (b)} Output from a coarse-to-fine human pose cascade
    (Section \ref{sec:pose}). The colored areas indicate valid 2D
    locations for each joint. Unlike the sequence cascade (a), the
    cliques stay the same from one layer to another. Instead, the
    resolution of the state space doubles with each additional layer.}
  \label{fig:cascade-output-example}
\end{figure}

\subsection{Cascaded inference with max-marginals}

In order to filter low-scoring outputs, we use {\em max-marginals}, for reasons
that we detail below.
For any value of $y_c$, we define the max-marginal $\mmarg{y_c}$ to be
the maximum score of any output $y$ that is consistent with the assignment 
$y_c$:
\begin{equation}
  \label{eq:mmarg_def}
  \mmarg{y_c} \triangleq \max_{y': y'_c = y_c} \score{y'}.
\end{equation}
Max-marginals can be computed exactly and efficiently for any clique
$c$ in low-treewidth graphs, although the computational
cost  is exponential in $|c|$ (the number of variables in the
clique) when the state-space is not filtered. 
Note that max-marginals can be computed over {\em any} clique
$c$, not just the cliques used in the feature function $\bft$; for
example, in Section \ref{sec:pose}, we compute max-marginals over
single variables (i.e., $y_c = y_j$ when performing human pose estimation 
(Figure \ref{fig:pose_overview}), but at
increasingly higher resolutions. On the other hand, in Section
\ref{sec:sequences} we compute max-marginals over increasingly large
cliques for sequence models (e.g. bigram, trigrams, and quadgrams).

\begin{figure}[t]
  \centering
  \includegraphics[width=0.8\textwidth]{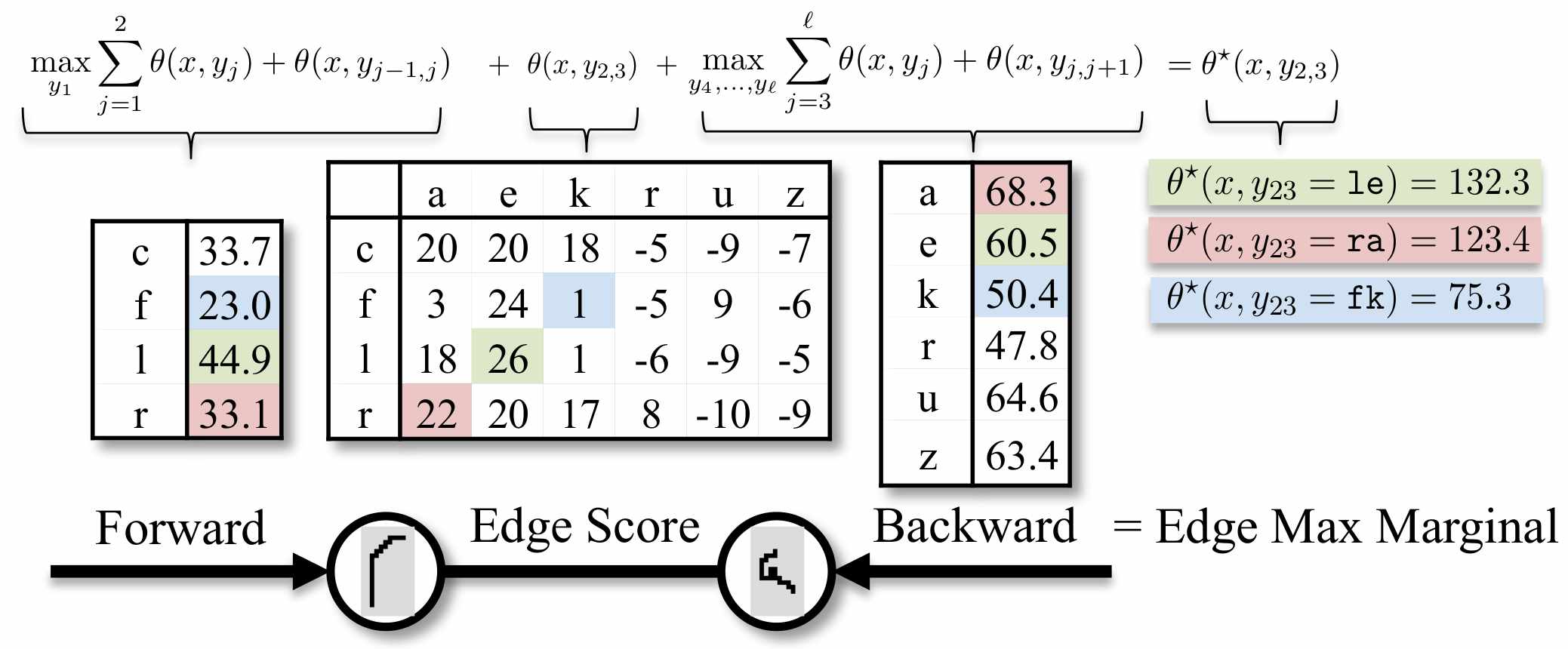}  
  \caption{\small Computing max-marginals over bigrams via message
    passing. The input is the same as in Figure
    \ref{fig:cascade-output-example}. Once forward and backward
    messages have been computed, the max-marginal is simply the sum of
    incoming messages and the score of the clique over bigrams.}
  \label{fig:computing-max-marginals}
\end{figure}

Exact computation of max-marginals for a clique $c$ requires the same amount of 
time to run as standard exact MAP inference.
This process is visualized in Figure \ref{fig:computing-max-marginals}: once forward and backward max-sum messages
have been computed for MAP inference, the max-marginal for a
given value $y_c$ is simply the sum of the score $\score{y_c}$ plus
the incoming messages to the variables in $c$. Note that in practice,
both stages of computation become faster as the output space becomes
increasingly sparse as the input proceeds through the cascade.  This
algorithm can also compute the maximizing assignment for each $y_c$,
\begin{equation}
  \label{eq:witness_def}
  \wit{y_c} \triangleq \argmax_{y': y'_c = y_c}\score{y'}.
\end{equation}
We call $\wit{y_c}$ the {\em argmax-marginal} or {\em witness} for
$y_c$ (it might not be unique, so we break ties in an arbitrary but deterministic way).

\begin{figure}[h]
  \centering
  \includegraphics[width=.6\textwidth]{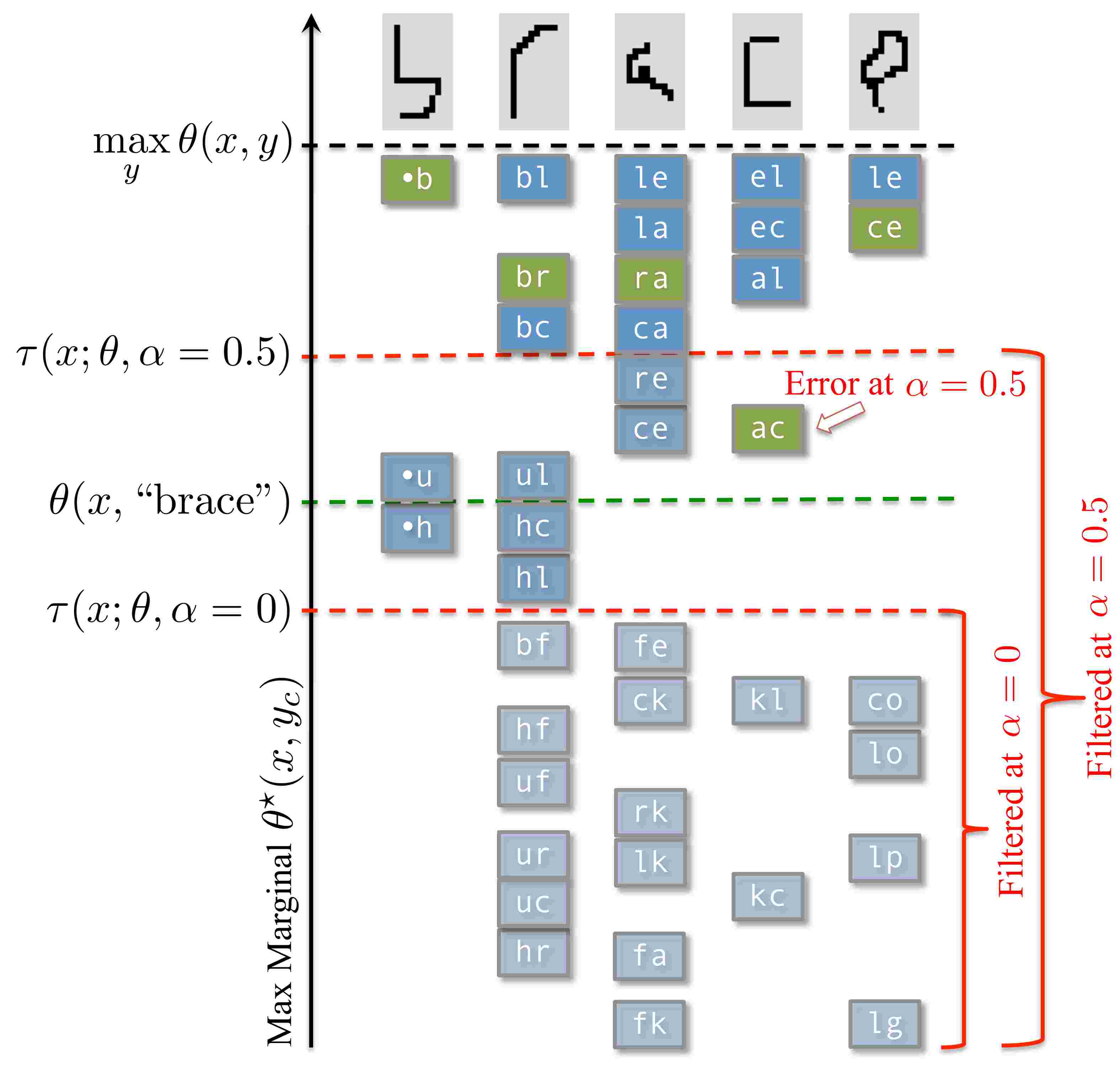}
  \caption{\small Thresholding bigrams using max-marginals. The input
    is the same as in Figure \ref{fig:cascade-output-example}. The
    sparse set of unfiltered bigrams is shown at each position
    according to the max-marginal score. The bigrams corresponding to
    the correct label sequence, {\tt brace}, are highlighted in green. The green
    dashed line indicates the score of the correct label sequence. Note
    that the max-marginals of the correct sequence are at least the score
    of the correct sequence. The black dashed line indicates the maximum
    score of any sequence, which is the maximum filtering
    threshold. The largest max-marginal values are all exactly equal
    to this score.  The red dashed lines indicate two candidate
    filtering thresholds $\thr=0$ and $\thr=0.5$ and corresponding sets of
    filtered bigrams are highlighted.
 Note that a filtering error occurs at the more 
    aggressive level of $\alpha=0.5$.
\out{    The threshold is exactly equal to the mean of the max
    marginals, filtering away roughly 50\% of possible
    bigrams. Because the true score (green line) is above the
    threshold, is it not possible for a mistake to occur. \textbf{b:}
    Filtering with $\alpha=0.5$. Because the threshold is above the
    score of the true sequence, filtering errors are now possible, and
    the bigram {\tt ac} is improperly pruned.}
    }
  \label{fig:threshold-example}
\end{figure}

Once max-marginals have been computed, we filter the output space by
discarding any clique assignments $y_c$ for which $\mmarg{y_j} \le t$
for a threshold $t$ (Figure \ref{fig:threshold-example}). This filtering rule has two desirable properties
for the cascade that follow immediately from the definition of
max-marginals:
\begin{lem} [Safe Filtering]
  \label{prop:pruning}
  If $\score{y} > t$, then $\forall c \;\; \mmarg{y_c} > t$.
\end{lem}
\begin{lem} [Safe Lattices]
  \label{prop:lattices}
  If $\max_{y'} \score{y'} > t$, then $\exists y 
  \;\forall c \;\; \mmarg{y_c} > t$.
\end{lem}
By Lemma \ref{prop:pruning}, ensuring that the score of the true label
$\score{y}$ is greater than the threshold is sufficient (although not
necessary) to guarantee that no marginal assignment $y_{c}$ consistent with the true
global assignment $y$ will be filtered. This condition will allow us to define a
max-marginal based loss function that we propose to optimize in Section \ref{sec:learning}
and will analyze in Section \ref{sec:theory}. Lemma
\ref{prop:lattices} follows from Lemma \ref{prop:pruning}, which
states that so long as the threshold is less than the maximizing
score, there always exists a global assignment $y$ with no pruned cliques
(i.e., a valid assignment always exists after pruning). Thus, Lemma
\ref{prop:lattices} guarantees that $|\cS^{i+1}| \ge 1$ in the \SPC
algorithm introduced above, and therefore  the cascade will
always produce a valid output. Note that neither property generally
holds for standard sum-product marginals $p(y_c|x)$ of a log-linear CRF (where
$p(y|x)\propto e^{\score{y}}$), which motivates our use of
max-marginals.

The next component of the inference procedure is choosing a threshold
$t$ for a given input $x$ (Figure \ref{fig:threshold-example}).  Note
that the threshold cannot be defined as a single global value but
should instead depend strongly on the input $x$ and $\score{\cdot}$
since scores are on different scales for different $x$. We also have
the constraint that computing a threshold function must be fast enough
such that sequentially computing scores and thresholds for multiple
models in the cascade does not adversely effect the efficiency of the
whole procedure. One might choose a quantile function to consistently
eliminate a desired proportion of the max-marginals for each
example. However, quantile functions are discontinuous in $\theta$
function, and we instead approximate a quantile threshold with a
threshold function that is continuous and convex in $\theta$.  We call this
the  {\em max-mean-max} threshold function (Figure
\ref{fig:threshold-example}), and define it as a convex combination of  the maximum score
and the mean
of the max-marginals:
 \begin{equation}
   \label{eq:meanmax}
   \thr = \alpha \max_y \score{y} + (1-\alpha)\frac{1}{ \sum_{c \in C}|\cY_c| }\sum_{c \in C}\sum_{y_c \in \cY_c} \mmarg{y_c}.
 \end{equation}
 Choosing a threshold using \eqref{eq:meanmax} is therefore equivalent
 to picking a $\alpha \in [0,1)$. Note that $\thr$ is a convex
 function of $\theta$ (in fact, piece-wise linear), which
 combined with Lemma \ref{prop:pruning} will be important for learning
 the filtering models and analyzing their generalization. In our
 experiments, we found that the distribution of max-marginals was well
 centered around the mean, so that choosing $\alpha \approx 0$
 resulted in $\approx50\%$ of max-marginals being eliminated on
 average. As $\alpha$ approaches $1$, the number of max-marginals
 eliminated rapidly approaches $100\%$.\footnote{We use
   cross-validation to determine the optimal $\alpha$ in our
   experiments (see Section \ref{sec:applications}).}

 In summary, the inner loop of the \SPC algorithm can be detailed
 as follows. The sparse output space $\cS^i$ is a list of valid
 assignments $y_c$ for each clique $c$ in the model $\bft_i$ (e.g.,
 Figure \ref{fig:cascade-output-example}):
 \begin{equation}
   \cS^i = \left\{ \cY_c \mid \forall c \in \cC \right\}\qquad \textrm{(list of valid clique values for all cliques)}
 \end{equation}
 Next, sparse max-sum message passing is used to compute max-marginals
$\mmarg{y_c}$  \eqref{eq:mmarg_def} for each value $y_c \in \cY_c$ of each
 clique $c$ of interest. Finally, for a given $\alpha$, a threshold is
 computed and low-scoring values of $\mmarg{y_c}$ are
 eliminated. Depending on the model in the next layer of the cascade,
 further transformation of the states may be necessary: For example,
 in the coarse-to-fine pose cascade (Section \ref{sec:pose}), valid
 2-D locations for each limb are halved either vertically or
 horizontally to produce finer-resolution states for the next model
 (Figure \ref{fig:cascade-output-example}b).


\subsection{Cascaded Inference in Loopy Graphs}
\label{sec:sc-loopy}

Thus far, we have assumed that (sparse) inference is feasible, so that
max marginals can be computed. In this section, we describe how to
apply \SPC when exact max-sum message passing is computationally
infeasible due to loops in the graph structure of the model. In order
to simplify the presentation in this section, we will assume that the
structured cascade under consideration operates in a ``node-centric''
coarse-to-fine manner as follows: For each variable $y_j$ in the
model, each level of the cascade filters a current set of possible
states $\cY_j$, and any surviving states are passed forward to the
next level of the cascade by substituting each state with its set of
descendents in a hierarchy. For example, such hierarchies arise in
pose estimation (Section \ref{sec:pose}) by discretizing the
articulation of joints at multiple resolutions, or in image
segmentation due to the semantic relationship between class labels
(e.g., ``grass'' and ``tree'' can be grouped as ``plants,'' ``horse''
and ``cow'' can be grouped as ``animal.'') Thus, in the pose
estimation problem, surviving states are subdivided into multiple
finer-resolution states; in the image segmentation problem, broader
object classes are split into their constituent classes for the next
level.

The key idea of this section is that we decompose the loopy model into
a collection of equivalent tractable sub-models for which inference is
tractable. What distinguishes this approach from other decomposition
based methods (e.g., \citet{komodakis2007dualdecomp,bertsekas99}) is
that, because the cascade's objective is filtering and not decoding,
our approach does not require enforcing the constraint that the
sub-models agree on which output has maximum score. In preliminary
work \citep{weiss10ensemble}, this approach was called {\em structured
  ensemble cascades}, here we simply refer to it as Ensemble-\SPC.

Given a loopy (intractable) graphical model, it is always possible to
express the score of a given output $\score{y}$ as the sum of $P$
scores $\pscore{y}$ under sub-models that collectively cover every
edge in the loopy model: $\score{y} = \sum_p \pscore{y}$ (Figure
\ref{fig:comb-example}). However, it is {\em not} the case that
optimizing each individual sub-model separately will yield the single
globally optimum solution. Instead, care must be taken to enforce
agreement between sub-models. For example, in the method of dual
decomposition \citep{komodakis2007dualdecomp}, it is possible to solve
a relaxed MAP problem in the (intractable) full model by running
inference in the (tractable) sub-models under the constraint that {\em
  all sub-models agree on the argmax solution.}  Enforcing this
constraint requires iteratively re-weighting unary potentials of the
sub-models and repeatedly re-running inference until each sub-model
convergences to the same argmax solution.

\begin{figure}[t]
  \centering
  \includegraphics[width=0.7\textwidth]{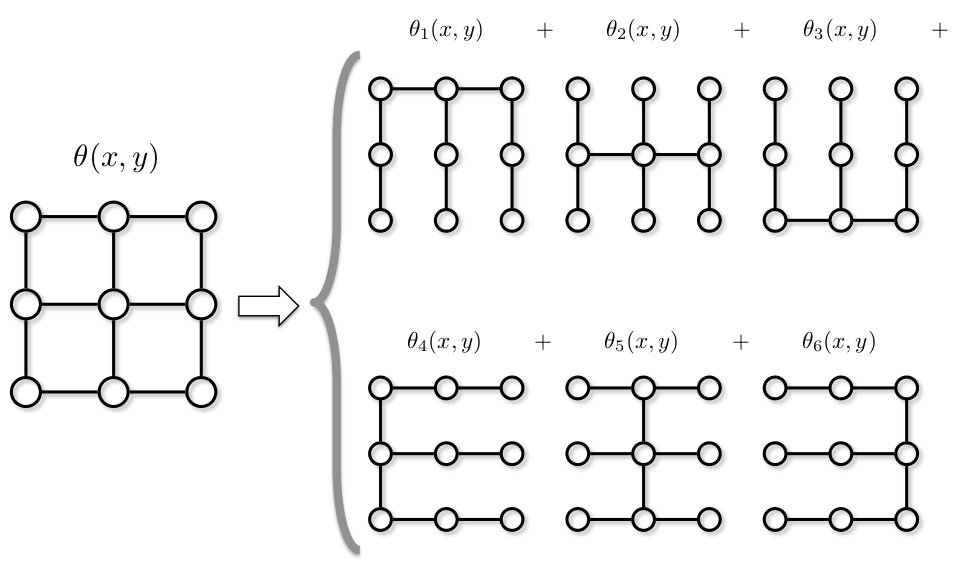}
  \caption{Example decomposition of a $3 \times 3$ fully connected
    grid into all six constituent ``comb'' trees. In general, a $n
    \times n$ grid yields $2n$ such trees.}
  \label{fig:comb-example}
\end{figure}

However, for the purposes of \SPC, we are only interested in computing
the max-marginals $\mmarg{y_j}$. In other words, we are only
interested in knowing whether or not a configuration $y$ consistent
with $y_j$ that scores highly in each sub-model $\pscore{y}$ {\em
  exists.} We show in the remainder of this section that the
requirement that a {\em single} $y$ consistent with $y_j$ optimizes
the score of each submodel (i.e, that all sub-models {\em agree}) is
not necessary for the purposes of filtering. Thus, because we do not
have to enforce agreement between sub-models, we can apply \SPC to
intractable (loopy) models, but pay only a linear (factor of $P$)
increase in inference time over the tractable sub-models.

Formally, we define a single level of the Ensemble-\SPC as a set of
$P$ models such that $\score{y} = \sum_p \pscore{y}$.  We let
$\pmmarg{y_{c}}$, $\pscoremax$ and $\pthr$ denote the
max-marginals, max score, and threshold of the $p$'th model,
respectively. Recall that the {\em argmax-marginal} or {\em witness}
$\pwit{y_j}$ is defined as the maximizing complete assignment of the
corresponding max-marginal $\pmmarg{y_j}$. Then we have that
\begin{eqnarray}
  \label{eq:mmarg-equality}
  \mmarg{y_j} & = & \sum_p \pmmarg{y_j}  \qquad \textrm{(with agreement: $y = \pwit{y_j}$, $\forall p$)} \\
  \mmarg{y_j} & \le & \sum_p \pmmarg{y_j}  \qquad \textrm{(in general)}
\end{eqnarray}
Note that if we do not require the sub-models to agree, then
$\mmarg{y_j}$ is strictly less than $\sum_p \pmmarg{y_j}$. Nonetheless,
as we show next, the approximation $\mmarg{y_j} \approx \sum_p
\pmmarg{y_j}$ is still useful and sufficient for filtering in a
structured cascade.

We now show that if a given label $y$ has a high score in the full
model, it must also have a large ensemble max-marginal score, even if
the sub-models do not agree on the argmax. This extends Lemma
\ref{prop:pruning} for the ensemble case, as follows:
\begin{lem}[Joint Safe Filtering]
  \label{lem:joint_safe}
  If $\sum_p \pscore{y} > t$, then $\sum_p \pmmarg{y_j} > t$ for all $j$.
\end{lem}
\begin{proof}
  In English, this lemma states that if the global score is above a
  given threshold, then the sum of sub-model max-marginals is also
  above threshold (with no agreement constraint). The proof is
  straightforward.  For any $y_j$ consistent with $y$, we have
  $\pmmarg{y_j} \ge \pscore{y}$. Therefore $\sum_p \pmmarg{y_j} \ge
  \sum_p \pscore{y} > t$.
\end{proof}
Therefore, we see that an agreement constraint is not necessary in
order to filter safely: if we ensure that the combined score $\sum_p
\pscore{y}$ of the true label $y$ is above threshold, then we can
filter without making a mistake if we compute max-marginals by running
inference separately for each sub-model. However, there is still
potentially a price to pay for disagreement. If the sub-models do not
agree, {\em and} the truth is not above threshold, then the threshold
may filter {\em all} of the states for a given variable $y_j$ and
therefore ``break'' the cascade. This results from the fact that
without agreement, there is no single argmax output $y^\star$ that is
always above threshold for any $\alpha$; therefore, we do not have an
equivalent to Lemma \ref{prop:lattices} for the ensemble
case. However, we note that in our experiments (Section
\ref{sec:videopose}), we never experienced such breakdown of the
cascades.



\section{Learning Structured Prediction Cascades}
\label{sec:learning}

When learning a cascade, we have two competing objectives that we must trade off:
\begin{itemize}
\item {\bf Accuracy:} Minimize the number of errors incurred by each
  level of the cascade to ensure an accurate inference process in
  subsequent models.
\item {\bf Efficiency:} Maximize the number of filtered max-marginals at each
level in the cascade to ensure an efficient inference process in
subsequent models.
\end{itemize}
Given a training set, we can measure the accuracy and efficiency of
our cascade, but what is unknown is the performance of the cascade on
test data. In section \ref{sec:theory}, we provide a
guarantee that our estimates of accuracy and efficiency will be
reasonably close to the true performance measures with high
probability. This suggests that optimizing parameters to achieve a
desired trade-off on training data is a good idea.

We begin by quantifying accuracy and efficiency in terms of
max-marginals, as used by \SPC.  We define the {\em filtering loss}
$\cL_f$ to be a 0-1 loss indicating a mistakenly eliminated correct
assignment. As discussed in the previous section, Lemma
\ref{prop:pruning} states that an error can only occur if $\score{y}
\le \thr$. We also define the {\em efficiency loss} $\cL_e$ to simply
be the proportion of unfiltered clique assignments.

\begin{defn}[Filtering loss] A filtering error occurs when a
  max-marginal of a clique assignment of the correct output $y$ is pruned. We 
define filtering loss as \begin{equation}
    \label{eq:filter-loss}
    \cL_f(x,y;\theta,\alpha) = \ind{\score{y}\le \thr}.  
  \end{equation}
  
\end{defn}
\begin{defn}[Efficiency loss] The efficiency loss is the proportion of
  unpruned clique assignments:
  \begin{equation}
    \label{eq:eff-loss}
    \cL_e(x,y;\theta,\alpha) =
    \frac{1}{\sum_{c \in \cC} |\cY_c|} \sum_{c\in C, y_c \in \cY_c} \ind{\mmarg{y_c} >
      \thr}.
  \end{equation}
\end{defn}

We now turn to the problem of learning parameters $\theta$ and
tuning of the threshold parameter $\alpha$ from training data. We have two
competing objectives, accuracy ($\cL_f$) and efficiency ($\cL_e$),
that we must trade off. Note that we can trivially minimize either of
these at the expense of maximizing the other. If we set
($\theta,\alpha$) to achieve a minimal threshold such that no
assignments are ever filtered, then $\cL_f = 0$ and $\cL_e =
1$. Alternatively, if we choose a threshold to filter every
assignment, then $\cL_f = 1$ while $\cL_e = 0$. To learn a cascade of
practical value, we can minimize one loss while constraining the other
below a fixed level $\epsilon$. Since the ultimate goal of the cascade
is accurate classification, we focus on the problem of minimizing
efficiency loss while constraining the filtering loss to be below a
desired tolerance.

We express the cascade learning objective for a {\em single level} of
the cascade as a joint optimization over $\theta$ and $\alpha$:
\begin{equation}
  \label{eq:cascade-learning}
  \min_{\theta,\alpha}\;\; \bbE_{X,Y}\left[ \cL_e(X,Y;\theta,\alpha) \right] \textrm{ s.t. } 
  \bbE_{X,Y}\left[ \cL_f(X,Y;\theta,\alpha)\right] \le \epsilon.
\end{equation}

We solve this problem with for a single level of the cascade as
follows. First, we define a convex upper-bound \eqref{eq:convex-opt}
on the filter error $\cL_f$, making the problem of minimizing $\cL_f$
convex in $\theta$ (given $\alpha$). We learn $\theta$ to minimize
filter error for several settings of $\alpha$ (thus controlling
filtering efficiency). Given several possible values for $\theta$, we
optimize the objective \eqref{eq:cascade-learning} over $\alpha$
directly, using estimates of $\cL_f$ and $\cL_e$ computed on a
held-out development set, and choose the best $\theta$. Note that in
Section \ref{sec:theory}, we present a theorem bounding the deviation
of our estimates of the efficiency and filtering loss from the
expectation of these losses.

\begin{algorithm}[tb]
  \caption{Forward Batch Learning of Structured Prediction
    Cascades.}
  \label{alg:learning}
  \begin{algorithmic}
    \REQUIRE Data $\{(x^i, y^i)\}_1^n$, structured feature generators
    $\bft^0,\dots, \bft^T$ and parameters $\alpha^0, \dots, \alpha^{T-1}$.

    \ENSURE Cascade parameters $\theta^0, \dots, \theta^T$.

    \STATE {\bf Initialize} $\cS^0(x^i) = \cY(x^i)$ for
    each example.  

    \FOR{$t = 0$ {\bfseries to} $T-1$} 

    \STATE $\bullet$ Optimize \eqref{eq:convex-opt} with sparse
    inference over the valid set $\cS^t$ to find $\theta^t$. 
    
    \STATE $\bullet$ Generate $\cS^{t+1}(x^i)$ from $\cS^t(x^i)$ by
    filtering low-scoring clique assignments $y_c$ where
    $$\mmargit{y_c} \le \tau(x^i; \theta^t,  \alpha^t)$$
    \ENDFOR

    $\bullet$ Learn $\theta^T$ using structured predictor over sparse state spaces $\cS^T(x^i)$.
  \end{algorithmic}
\end{algorithm}

For the first step of learning a single level of the cascade, we learn
the parameters $\theta$ for a fixed $\alpha$ using the following
convex margin optimization problem:
\begin{equation}
  \label{eq:convex-opt}
SPC:\;\;\;\;  \min_{\theta} \;\;\;\ \frac{\lambda}{2}||\theta||^2 + \frac{1}{n}\sum_i H(x^i,y^i;\theta,\alpha),
\end{equation}
where $H$ is a convex upper bound on the filter loss $\cL_f$,
\begin{equation*}
  H( x^i,y^i;\theta,\alpha) = \max\{0, \ell + \thri - \theta(x^i,y^i)\}.
\end{equation*}
The upper-bound $H$ is a hinge loss measuring the margin between the
filter threshold $\thri$ and the score of the truth
$\theta^\top\bft(x^i,y^i)$; the loss is zero if the truth scores above
the threshold by margin $\ell$ (in practice, the length $\ell$ can
vary by example).  We solve \eqref{eq:convex-opt} using stochastic
sub-gradient descent. Given a sample $(x,y)$, we apply the following
update if $H(\theta, x,y)$ (i.e., the sub-gradient) is non-zero:
\begin{equation}
    \label{eq:spf_update}
    \theta' \leftarrow (1-\eta\lambda)\theta + \eta \bft(x,y) - \eta \alpha 
\bft(x,y^\star)
    - \eta(1-\alpha) \frac{1}{\sum_c |\cY_c| } \sum_{c \in \cC, y_c \in \cY} \bft(x,\wit{y_c}).
\end{equation}
Above, $\eta$ is a learning rate parameter. The key distinguishing
feature of this update compared to the structured perceptron update is
that it subtracts features included in all max-marginal assignments
$\wit{y_c}$.


Note that because \eqref{eq:convex-opt} is $\lambda$-strongly convex,
we can choose $\eta_t = 1/(\lambda t)$ and add a projection step to keep 
$\theta$ in a fixed norm-ball. The update then corresponds to the Pegasos
update with convergence guarantees of $\tilde{O}(1/\epsilon)$ iterations for
$\epsilon$-accurate solutions~\citep{shalev07}.

An overview of the entire learning process for the whole cascade is
given in Algorithm \ref{alg:learning}. Levels of the cascade are
learned incrementally using the output of the previous level of the
cascade as input. Note that Algorithm \ref{alg:learning} trades memory
efficiency for time efficiency by storing the sparse data structures
$\cS^t$ for each example. A more memory-efficient (but less time
efficient) algorithm would instead run all previous layers of the
cascade for each example during sub-gradient descent optimization of
\eqref{eq:convex-opt}.

Finally, in our implementation, we can sometimes achieve better
results by further tuning the threshold parameters $\alpha^t$ using a
development set. We first learn $\theta^{t}$ using some fixed
$\alpha^t$ as before. However, we then choose an improved
$\bar{\alpha}^t$ by maximizing efficiency subject to the constraint
that filter loss on the development set is less than a {\em tolerance}
$\epsilon_t$:
$$\bar{\alpha}^t \leftarrow \argmin_{0 \le \alpha' < 1} \sum_{i=1}^n \cL_e(x^i,y^i;\theta^{t},\alpha') 
\st \frac{1}{n} \sum_{i=1}^n \cL_f(x^i,y^i;\theta^{t},\alpha') \le
\epsilon_t.$$ 
Furthermore, we can repeat this tuning process for
several different starting values of $\alpha^t$ and pick the
$(\theta^t, \bar{\alpha}^t)$ pair with the optimal trade-off, to
further improve performance. In practice, we find that this procedure
can substantially improve the efficiency of the cascade while keeping
accuracy within range of the given tolerance. 

\out{\todo{Should we also
  mention that we use partitions of the data like in Stacked
  Generalization (Wolpert, 1992) to make sure that data is not re-used
  in the training process?}}

It is straightforward to adapt Algorithm \ref{alg:learning} for the
Ensemble-\SPC case. As in the previous section, we first define the
natural loss function for sums of max-marginals, as suggested by Lemma
\ref{lem:joint_safe}. We define the {\em joint filtering loss} as
follows,
\begin{defn}[Joint Filtering Loss]
  \begin{equation}
    \label{eq:joint_loss}
    \cL_{joint}(x,y;\theta,\alpha) =  \ind{\sum_p\pscore{y} \le \sum_p \pthr}.
  \end{equation}
\end{defn}
We now discuss how to minimize the {\em joint} filter
loss \eqref{eq:joint_loss} given a dataset. We rephrase the \SPC
optimization problem \eqref{eq:convex-opt} using the ensemble
max-marginals to form the ensemble cascade margin problem,
\begin{equation}
  \label{eq:joint_slack}
  \min_{\theta_1, \dots, \theta_P, \xi \ge 0} \frac{\lambda}{2} \sum_p ||\theta_p||^2 + \frac{1}{n}\sum_i \xi^i
  \suchthat \sum_p \pscorei{y^i} \ge \sum_p \pthri + \ell^i - \xi^i.
\end{equation}
Seeing that the constraints can be ordered to show $\xi^i \le \sum_p
\pthri -\sum_p \pscorei{y^i} +\ell^i$, we can form an equivalent
unconstrained minimization problem,
\begin{equation}
  \label{eq:joint_hinge}
  \min_{\theta_1, \dots, \theta_P} \frac{\lambda}{2}\sum_p ||\theta_p||^2 + \frac{1}{n}\sum_{i} \hinge{\sum_p \pthri - \pscorei{y^i} + \ell^i},
\end{equation}
where $\hinge{z} = \max\{z, 0\}$. Finally, we
take the subgradient of the objective in \eqref{eq:joint_hinge} with respect to each
parameter $\theta_p$. This yields the following update rule for the
$p$'th model:
\begin{equation}
  \label{eq:joint_update}
  \theta_p \leftarrow (1-\lambda)\theta_p +
  \begin{cases}
    0 & \textrm{ if } \sum_p \pscorei{y^i} \ge \sum_p \pthri + \ell^i, \\
    \nabla\pscorei{y^i} - \nabla\pthri & \textrm{ otherwise.}
  \end{cases}
\end{equation}
This update is identical to the original \SPC update with the
exception that we update each model individually only when the
ensemble has made a mistake {\em jointly}. Thus, learning to filter
with the ensemble requires only $P$ times as many resources as
learning to filter with any of the models individually. We simply
replace the optimization over \eqref{eq:convex-opt} step in Algorithm
\ref{alg:learning} with an optimization over \eqref{eq:joint_hinge}.


\section{Generalization Analysis}
\label{sec:theory}

We now present generalization bounds on the filtering and efficiency loss
functions for a single level of a cascade.  To achieve bounds on the entire
cascade, these  can be combined provided that a fresh sample is
used for each level.  To prove the following bounds, we
make use of  Gaussian complexity results from
\cite{bartlettM02}, which requires vectorizing scoring and loss functions in
a novel structured manner (details in Appendix \ref{sec:proofs}).
The main theorem in this section
depends on Lipschitz dominating cost functions $\cL_f^{\gamma}$ and
$\cL_e^{\gamma}$ that upper bound $\cL_f$ and $\cL_e$.
Note that as $\gamma \rightarrow 0$, we recover $ \cL_f$ and $ \cL_e$.

\begin{defn}[Margin-augmented losses]  We define  margin-augmented filtering and efficiency losses using the usual $\gamma$-margin function:
 \begin{eqnarray}
    r_{\gamma}(z) &=& 
    \begin{cases}
       1 & \textrm{if } z < 0 \\
       1-z/\gamma & \textrm{if } 0 \le z \le \gamma \\
        0 & \textrm{if } z > \gamma.
    \end{cases}\\
    \cL_f^{\gamma}(x,y;\theta,\alpha) &=& r_{\gamma}(\score{y}- \thr)\\  
    \cL_e^{\gamma}(x,y;\theta,\alpha) &=&  \frac{1}{\sum_{c \in \cC} |\cY_c|} \sum_{c\in C, y_c \in \cY_c} r_{\gamma}( \thr-\mmarg{y_c}).
 \end{eqnarray}
\end{defn}

\begin{thm}
  \label{thm:gencascade}
Fix $\alpha \in
  [0,1]$ and let
  $\Theta$ be the class of all scoring functions $\theta$ with
  $||\theta||_2 \le B$, let $|\cC|$ be the total number of cliques,
  $m=\sum_{c \in \cC} |\cY_c|$ be the total number of clique assignments,
  $||\bft_{c}(x,y_c)||_2 \le 1$ for all $x \in \cX,c\in\cC$ and $y_c\in\cY_{c}$. Then 
  there exists a constant $c$ such that for
  any integer $n$ and any $0 < \delta < 1$ with probability $1-\delta$
  over samples of size $n$, every $\theta \in \Theta$ satisfies:
  \begin{eqnarray}
    \label{eq:gencascade}
      \bbE \left[\cL_f(X,Y;\theta,\alpha)\right] &\le& \hat{\bbE}\left[\cL_f^{\gamma}(X,Y;\theta,\alpha)\right] +
      \frac{cmB\sqrt{\ncliques}}{\gamma\sqrt{n}} 
      + \sqrt{\frac{8 \ln (2/\delta)}{n}},\\
    \label{eq:gencascade-eff}
      \bbE \left[\cL_e(X,Y;\theta,\alpha)\right] &\le& \hat{\bbE}\left[\cL_e^{\gamma}(X,Y;\theta,\alpha)\right] +
      \frac{cmB\sqrt{\ncliques}}{\gamma\sqrt{n}} 
      + \sqrt{\frac{8 \ln (2/\delta)}{n}},
  \end{eqnarray}
  where $\hat{\bbE}$ is the empirical expectation with respect to
  the sample. 
\end{thm}

Theorem \ref{thm:gencascade} provides theoretical justification for
the definitions of the loss functions $\cL_e$ and $\cL_f$ and the
structured cascade objective; if we observe a highly accurate and
efficient filtering model $(\theta,\alpha)$ on a finite sample of
training data, it is likely that the performance of the model on
unseen test data will not be too much worse as $n$ gets large. Theorem
\ref{thm:gencascade} is the first theoretical guarantee on the
generalization of {\em accuracy} and {\em efficiency} of a structured
filtering model.

We now turn to ensemble setting and define an appropriate margin-augmented loss:

\begin{defn}[Ensemble margin-augmented loss]  
 \begin{eqnarray}
    \cL_{joint}^{\gamma}(x,y;\theta,\alpha) &=& r_{\gamma}\left(\sum_{p} \pscore{y}- \pthr\right)
 \end{eqnarray}
\end{defn}

\begin{thm}
  \label{thm:gencascade-ensemble}
  Fix $\alpha \in
  [0,1]$ and let $||\theta_p||_2 \le B/P$ for all $p$, and
  $||\bft_{c}(x,y_c)||_2 \le 1$ for all $x$ and $y_c$. Then there exists a
  constant $c$ such that for any integer $n$ and any $0 < \delta < 1$
  with probability $1-\delta$ over samples of size $n$, every $\theta
  = \{\theta_1, \dots, \theta_P\}$  satisfies:
  \begin{equation}
    \label{eq:gencascade-ens}
      \bbE \left[\cL_{joint}(X,Y;\theta,\alpha)\right] \le
      \hat{\bbE}\left[\cL_{joint}^{\gamma}(X,Y;\theta,\alpha)\right] +
      \frac{cmBP\sqrt{\ncliques}}{\gamma\sqrt{n}}
      + \sqrt{\frac{8 \ln (2/\delta)}{n}},
    \end{equation}
  where $\hat{\bbE}$ is the empirical expectation with respect to
  the sample.
\end{thm}
The proof of Theorem \ref{thm:gencascade-ensemble} is given in
Appendix \ref{sec:proofs}.

\section{Applications}
\label{sec:applications}

In this section, we explore in detail several evaluations of structured
prediction cascades. In Section \ref{sec:sequences} we describe a
structured prediction cascade for sequential data and apply this model
handwriting recognition. In Section \ref{sec:pose}, we
describe \SPC for articulated human pose estimation from single frame
images. Finally, in Section \ref{sec:videopose}, we evaluate
Ensemble-\SPC on a synthetic image segmentation task and the problem
of {\em detection and tracking} articulated poses in video.

\subsection{Linear-chain Cascade}
\label{sec:sequences}

In this section, we apply the structured prediction cascades framework
to sequence prediction tasks with increasingly high order linear-chain
models. The state space of a linear chain model is $\forall i \::
\cY_i = \{1, \dots, K\}$, where $K$ is the number of possible states.
Thus, the size of the state space is $K$. A $d$-order linear-chain
model has maximal cliques $\{x, y_i, y_{i-1}, \dots, y_{i-d}\}$. Thus,
for an order $d$ clique, there are $K^d$ possible clique assignments,
although we find that in practice very few high-order clique
assignments survive the first few levels of the cascade (see Table
\ref{tab:ocr}.)

For a $d$-order linear-chain model, the score of an output
$y$ is given by a combination of unary features and transition
features,
\begin{equation}
  \label{eq:linear-chain}
  \score{y} = \sum_{i=1}^\ell \theta_0^\top \bft_0(x, y_i)  + \sum_{i=1}^\ell \sum_{j=1}^d \theta_j^\top\bft_j(y_i, \; \dots \;, y_{i-j})
\end{equation}
where $\theta_0$ is a set of parameters for unary features
$\bft(x,y_i)$ that depend on a single output variable and $\theta_j$
is a set of parameters scoring $j$-order transition features
$\bft_j(y_i, \dots, y_{i-j})$. 


In general, any $d$-order linear-chain model can be equivalently
represented as a bigram (2-order) model with $K^{d-1}$ states. Thus,
it is simplest to implement a cascade of sequence models of increasing
order as a set of bigram models where the state space is increasing
exponentially by a factor of $K$ from one model to the next.  Given a
list of valid assignments $\cS^t$ in a $d$-order model, we can
generate an expanded list of valid assignments $\cS^{t+1}$ for a
$(d+1)$-order model by concatenating the valid $d$-grams with all
possible additional states.


\begin{figure}[t]
  \centering
  \includegraphics[width=\textwidth]{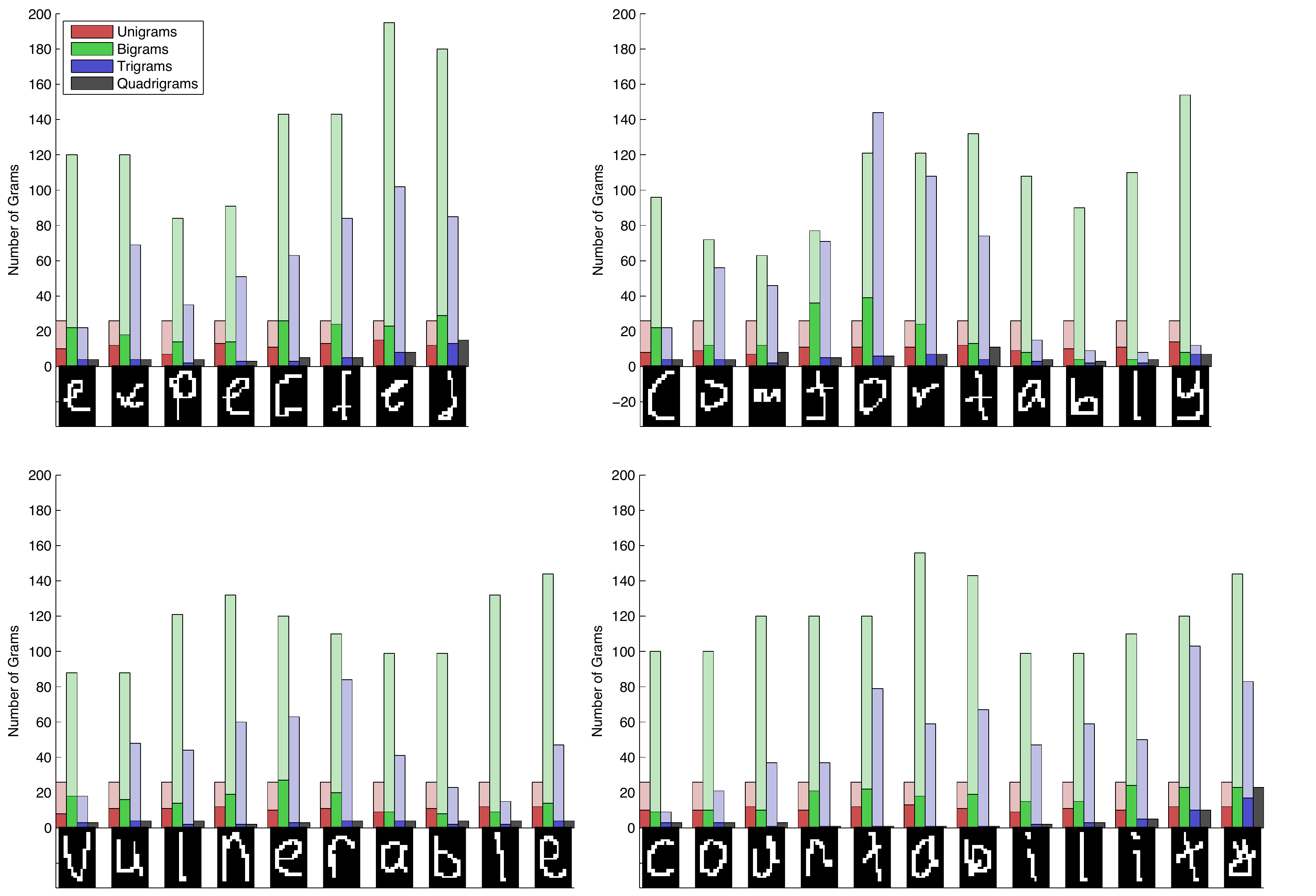}  
  \caption{Sparsity of inference during an example sequence
    cascade. Each panel shows the complexity of inference on a
    different example from the OCR dataset at each position in the
    sequence. The total height of each bar represents the size of the
    valid assignments $\cS^t$, while the shaded portion represents the
    remaining assignments after thresholding. Although complexity
    rises as unigrams are expanded into bigrams, filtering with
    bigrams and trigrams quickly reduces complexity to a few possible
    assignments at each position.}
  \label{fig:ocr-output}
\end{figure}

\subsubsection{Handwriting Recognition}

We first evaluated the accuracy of the cascade using the handwriting
recognition dataset from \citet{taskar03max}. This dataset consists of
6877 handwritten words, with average length of $\sim$8 characters,
from 150 human subjects, from the data set collected by
\citet{Kassel_Thesis}. Each word was segmented into characters, each
character was rasterized into an image of 16 by 8 binary pixels. The
dataset is divided into 10 folds; we used 9 folds for training and a
single withheld for testing (note that \citet{taskar03max} used 9
folds for testing and 1 for training due to computational limitations,
so our results are not directly comparable). Results are averaged
across all 10 folds.

Our objective was to measure the improvement in predictive accuracy as
higher order models were incorporated into the cascade. We trained six
cascades, up to a sixth-order (sexagram) linear-chain model. This is
significantly higher order than the typical third-order (trigram)
models typically used in sequence classification tasks. Note that in
practice, the additional accuracy gained by increasing the order of
the model might be offset by the additional filtering errors incurred
due to lengthening the cascade. Thus, each level of each cascade was
tuned to achieve maximum efficiency subject to a maximum error {\em
  tolerance} $\epsilon$, whereby $\alpha$ was set such that no more
than $\epsilon$ filtering error was incurred by each level of the
cascade.

Results are summarized in Table \ref{tab:ocr}. We found that using
higher order models led to a dramatic gain in accuracy on this
dataset, increasing character accuracy from 77.35\% to 98.54\% and
increasing word accuracy from 26.74\% to 96.16\%. It is interesting to
note that the word level accuracy of the sixth-order model is roughly
equivalent to the character-level accuracy of the trigram
model. Furthermore, using a development set, we found that a stricter
tolerance was required to gain accuracy from fifth- and sixth-order
models, as reflected in Table \ref{tab:ocr}. Finally, compared to
previous approaches on this dataset, our accuracies are much higher;
the best previously reported result on this dataset was 90.19\%
\citep{searn}.

In fact, the extremely high accuracies of our approach on this dataset
highlight the particular features of this data.  Due to the high
number of subjects used, there are only 55 unique words in the
handwriting recognition dataset. In fact, if just the first three
letters of each word are given exactly, one can guess the identity of
the word with 94.5\% accuracy. Given more letters, it is possible to
uniquely identify the word with 100\% accuracy. However, due to
inter-subject variance, previous approaches have not been able
approach this theoretical performance. By being able to utilize very
high order cliques, \SPC overcomes this limitation.

To gain intuition about the inference process of \SPC, a detailed
picture of the complexity of inference for a few representative
examples is presented in Figure \ref{fig:ocr-output} for the
fourth-order cascade model. This figure also demonstrates the
flexibility of the cascade: although a single threshold is chosen, the
max marginals around unambiguous portions of the input are eliminated
first.

\begin{table}
  \centering
  {\small
    \begin{tabular}{|r|cccccc|}
      \hline {\bf Model Order} & {\bf 1} & {\bf 2} & {\bf 3} & {\bf 4} &  {\bf 5}  & {\bf 6} \\
      \hline 
    Accuracy, Char. (\%) & 77.35  & 85.02 & 96.20 & 97.21 & 98.27 & 98.54 \\
    Accuracy, Word (\%) & 26.74 & 45.67 & 88.25 & 91.35 & 93.74 & 96.16 \\
    Filter Loss (\%) & --- & 0.50 & 0.73 & 1.00 & 0.75 & 0.57 \\
    Tolerance (\%) & 1.00 & 1.00 & 1.00 & 1.00 & 0.50 & 0.25 \\
    Avg. Num $n$-grams & 26.0 & 127.97 & 101.84 & 18.80 & 82.12 & 73.36 \\
    \hline
  \end{tabular}}
\caption{\small Summary of handwriting recognition results. For each level of the
  cascade, we computed prediction accuracy (at character and word
  levels) using a standard voting
  perceptron algorithm as well as the filtering loss and average
  number of unfiltered $n$-grams per position for the \SPC on the test set.}
  \label{tab:ocr}
\end{table}


\subsection{Pictorial Structure Cascade}
\label{sec:pose}

\begin{figure}[t]
  \centering
  \includegraphics[width=0.9\columnwidth]{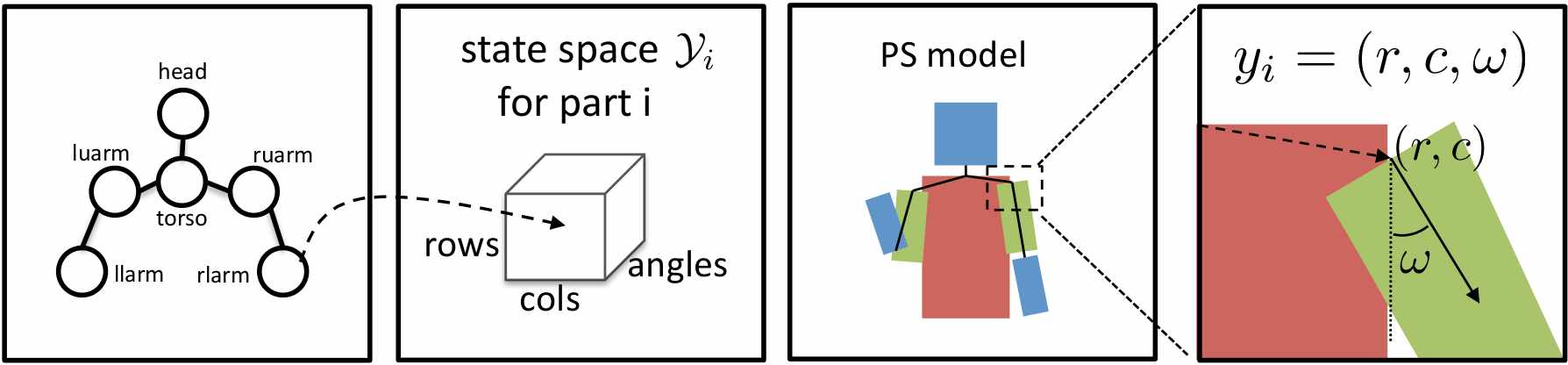}
  \caption{Basic PS model with state $\lpart_i$ for a part
    $\Lpart_i$. Left: graphical model representation.  Second: state space representation as a tensor.  Rightmost two panels: Illustration of state space laid out as a stick figure representation in image coordinates.}
  \label{fig:ps}
\end{figure}

Classical pictorial structures (PS) are a class of graphical models
where the nodes of the graph represents object parts, and edges
between parts encode pairwise geometric relationships. For modeling
human pose, the standard PS model is a tree structure with unary
potentials (also referred to as appearance terms) for each part and
pairwise terms between pairs of physically connected parts.
Figure~\ref{fig:ps} shows a PS model for 6 upper body parts, with
lower arms connected to upper arms, and upper arms and head connected
to torso.  Note that in previous
work~\citep{deva2006,felzps,ferrari08,andriluka09} (unlike
the approach described in this section), the pairwise terms do not
depend on data and are hence referred to as a ``spatial'' or
``structural'' prior.

The state of part $i$, denoted as $\lpart_i \in \Lspace_i$, encodes
the joint location of the part in image coordinates and the direction
of the limb as a unit vector: $\lpart_i = [\lpart_{ix} \; \lpart_{iy}
\; \lpart_{iu} \; \lpart_{iv}]^T$. The state of the model is the
collection of states of $\ell$ parts: $y = [\lpart_1, \ldots,
\lpart_\ell]$.  The size of the state space for each part,
$|\Lspace_i|$, is the number of possible locations in the image times
the number of pre-defined discretized angles. For example, we model
the state space of each part in a $80 \times 80$ grid for $\lpart_{ix}
\times \lpart_{iy}$, with 24 different possible values of angles,
yielding $|\Lspace_i| = 80 \times 80 \times 24 = 153,600$ possible placements.



Given a part configuration, we define cliques over pairwise and unary terms:
\begin{align}
\score{y} = \sum_{ij}  \theta_{ij}^T \bft_{ij}(x, \lpart_i,\lpart_j) +
\sum_i \theta_i^T\bft_i(x, \lpart_i)
\label{eqn:our_ps}
\end{align}
Thus, the parameters of the model are the pairwise and unary weight vectors
$\theta_{ij}$ and $\theta_i$ corresponding to the pairwise and unary
feature vectors $\bft_{ij}(x, \lpart_i,\lpart_j)$ and $\bft_i(x, \lpart_i)$.


One of the reasons pictorial structures models have been so popular in the 
literature is that~\citet{felzps} proposed a way to perform max inference on 
\eqref{eqn:our_ps} in linear time using distance transforms, which is only 
possible if the pairwise term is a quadratic function of the displacement 
between neighbors $y_i$ and $y_j$.  We wish to go beyond such a simple 
geometric prior for the pairwise term of~\eqref{eqn:our_ps}, and thus rely on 
standard $O(|\Lspace_i|^2)$ dynamic programming techniques to compute the MAP 
assignment or part posteriors, as was the case for linear-chain models in the 
previous section. However, unlike linear-chain models, many highly-effective
pairwise features one might design would be intractable to compute in
this manner for a reasonably-sized state space---for example an
$80\times80$ image with a part angle discretization of $24$ bins
yields $|\Lspace_i|^2 = 57.6$ billion part-part hypotheses, far too
many to store in a dynamic programming table (e.g., Figure
\ref{fig:computing-max-marginals}).


\subsubsection{Coarse-to-Fine Resolution Cascade}

\begin{figure}[t]
  \centering
\includegraphics[width=0.8\textwidth]{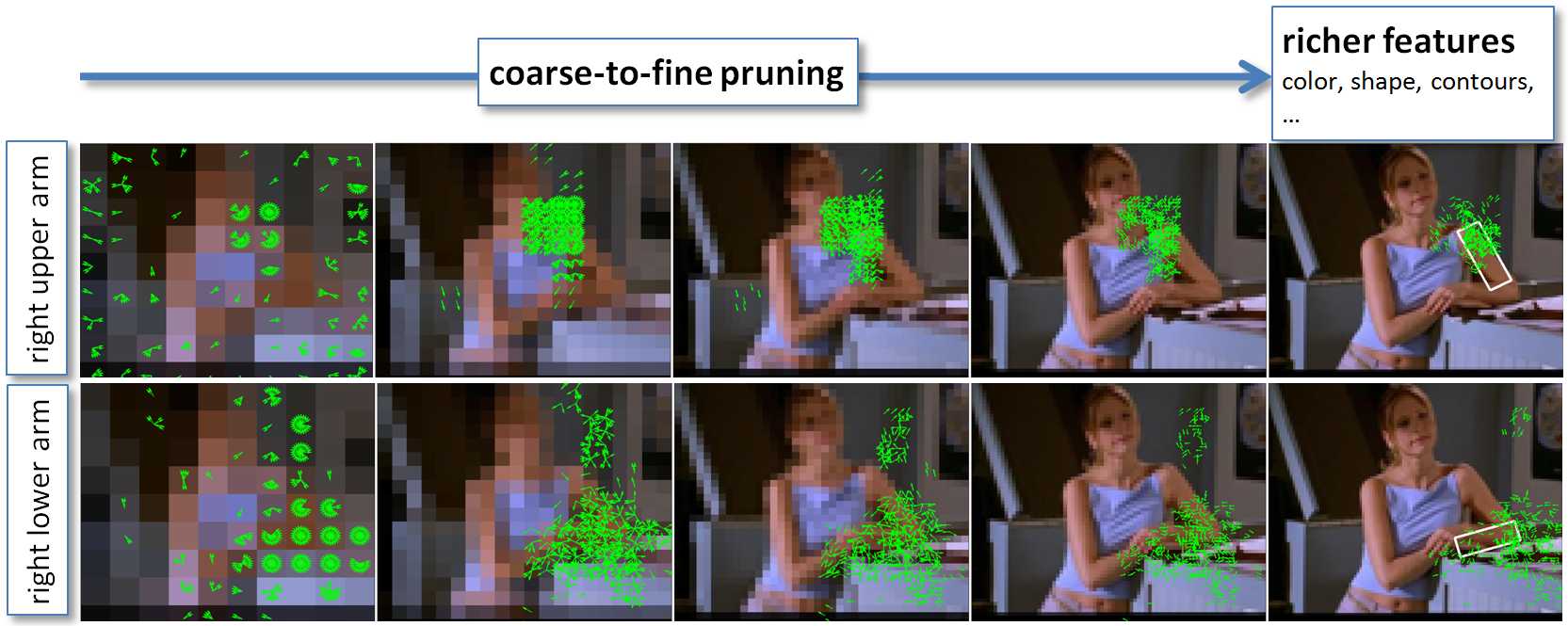}
\caption{\label{fig:pose_overview} Overview: A coarse-to-fine cascade
  of pictorial structures filters the pose space so that expressive
  and computationally expensive features can be used in the final
  pictorial structure.  Shown are 5 levels of the coarse-to-fine
  cascade for the right upper and lower arm parts.  Green vectors
  represent position and angle of unpruned states, the downsampled
  images correspond to the dimensions of the respective state space,
  and the white rectangles represent classification using our final
  model.}
\end{figure}

To overcome the issue of feature intractability, we define a {\em
  coarse-to-fine} structured prediction cascade over the resolution of the 
state space
$\Lspace_i$ (Figure \ref{fig:cascade-output-example}b). Note that
unlike the linear-chain cascade, the cliques do not change from one
level to the next. Instead, the state space $\Lspace_i$ of each part
in one model is subdivided to form the state space of the next
model. Once again, we learn parameters $\theta$ and $\alpha$ for the
cascade using Algorithm \ref{alg:learning}. The coarse-to-fine cascade
is outlined in Figure \ref{fig:pose_overview}.


Max-marginals for the pose model can be visualized to provide some
intuition for max marginals. In general, the max-marginal for
location/angle $\mmarg{\lpart_i}$ is the score of the best global body
pose which constrains $\Lpart_i = \lpart_i$. In a pictorial structure
model, this corresponds to fixing limb $i$ at (location, angle) $\lpart_i$, and
determining the highest scoring configuration of other part locations
and angles under this constraint. Thus, a part could
have weak individual image evidence of being at location $\lpart_i$
but still have a high max-marginal score if the rest of the model
believes this is a likely location.


While the fine-level, target state space has size $80 \times 80 \times 24$,
the first level cascade coarsens the state-space down to $10 \times 10
\times 12 = 1200$ states per part, which allows for efficient
exhaustive inference. In our experiments, we always set $\alpha = 0$,
effectively throwing away half of the states at each stage.  After
pruning we double one of the dimensions (first angle, then the minimum
of width or height) and continue (see Table~\ref{tab:pruning}).


\begin{figure}[t]
  \centering
  \includegraphics[width=0.65\textwidth]{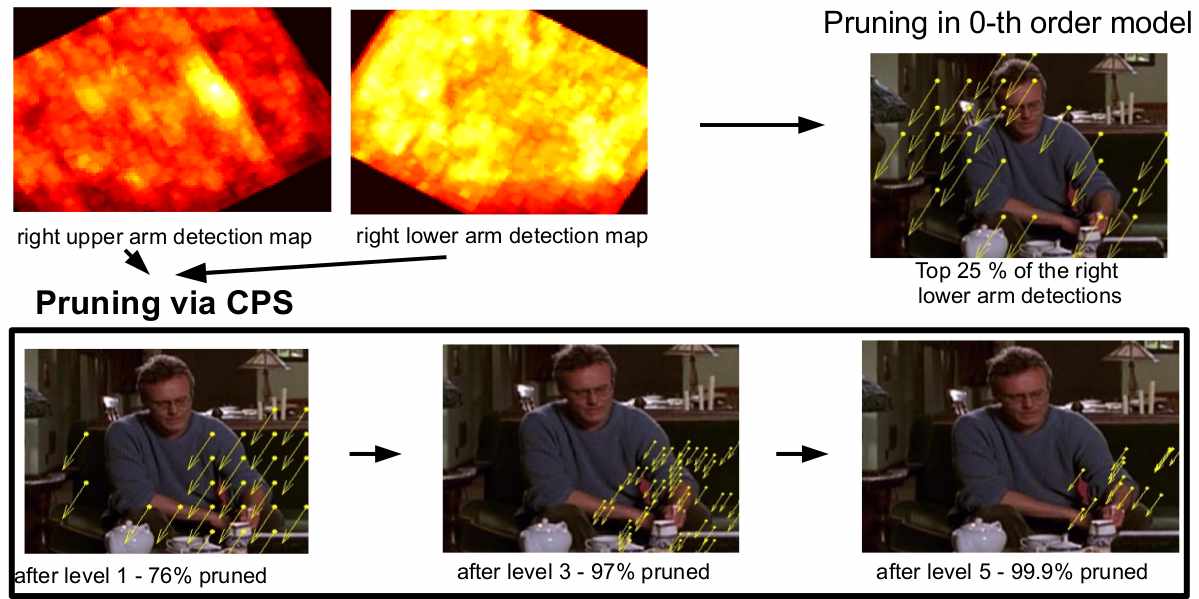}
  \includegraphics[width=0.3\textwidth]{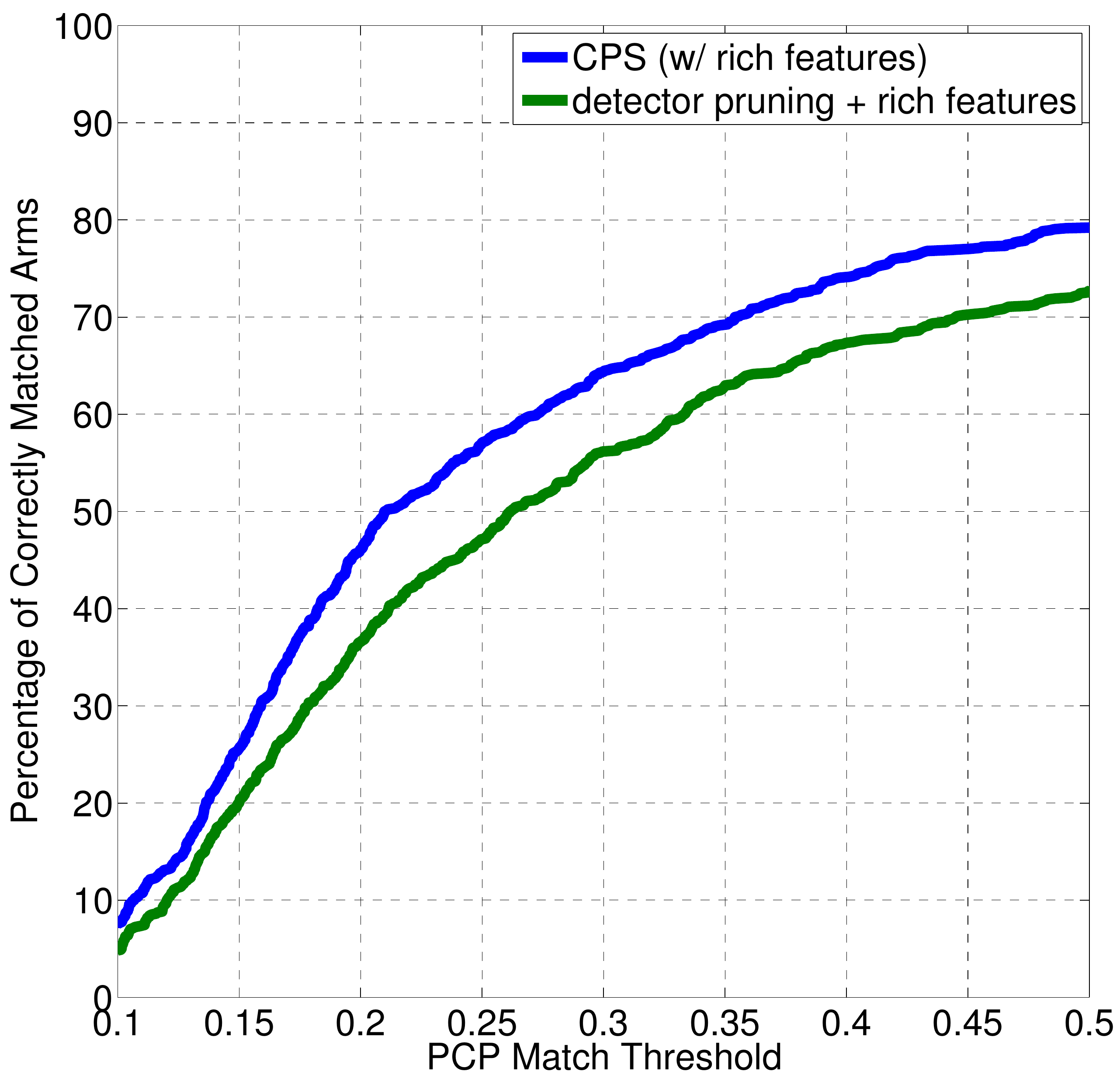}
  \caption{ {\bf Left:} Detector-based pruning (0th order model) by
    thresholding yields many hypotheses far way from the true one for
    the lower right arm. The CPS (bottom row), however, exploits
    global information to perform better pruning. {\bf Right:} PCP
    curves of our cascade method (blue) show increased accuracy versus
    a detection pruning approach (green), evaluated using PCP on arm
    parts.}
  \label{fig:cascade_pruning}
\end{figure}

The coarse-to-fine stages use standard PS features. HoG part detectors
are run once over the original state space, and the outputs are
resized for features in the coarser state spaces.  For pairwise
features, we use the standard relative geometric cues of angle and
displacement. The features are discretized uniformly, and thus
multi-modal pairwise costs can be learned. 


Once the cascade has reduced the fine-level state space to a manageable
size, we apply a boosted model with many expensive, powerful features.
As can be seen in Table~\ref{tab:pruning}, the coarse-to-fine cascade
leaves us with roughly 500 valid assignments per part; for each
possible part location and valid part pairs, we compute features
using image contours, moments of the shape and regions underlying each
part, color and texture appearance models, $\chi^2$ color similarity
between parts, and geometry.

One practical detail differentiates this cascade from others discussed
in this section. Rather than learn a standard structured perceptron
for prediction in the final stage, we concatenate all unary and
pairwise features for part-pairs into a feature vector and learn
boosting ensembles which give us our pairwise clique scores\footnote{
  We use OpenCV's implementation of Gentleboost and boost on trees of
  depth 3, setting the optimal number of rounds via a hold-out set.}.
This method of learning clique scores has several advantages over
stochastic subgradient learning: it is faster to train, can determine
better thresholds on features than uniform binning, and can combine
different features in a tree to learn complex, non-linear
interactions. In general, we can use any method of learning with
sparse inference for the final stage of Algorithm \ref{alg:learning}.



\subsubsection{Buffy and PASCAL Dataset Results}

We evaluated the pose cascade on the publicly available Buffy The
Vampire Slayer v2.1 and PASCAL Stickmen datasets~\citep{eichner09}. We
used the upper body detection windows provided with the dataset as
input to localize and scale normalize the images before running our
experiments as in~\citet{eichner09,ferrari08,andriluka09}.  The standard 235 
Buffy test images were used for testing, as well as the 360
detected people from PASCAL stickmen.  We used the remaining 513
images from Buffy for training and validation.

The typical measure of performance on this dataset is a matching
criteria based on both endpoints of each part (e.g., matching the
elbow and the wrist correctly): A limb guess is correct if the limb
endpoints are on average within $r$ of the corresponding groundtruth
segments, where $r$ is a fraction of the groundtruth part length.  By
varying $r$, a performance curve is produced where the performance is
measured in the percentage of correct parts (PCP) matched with respect
to $r$. We define PCP$_r$ as the value of the curve at $r$.  


\begin{table}[t]
\begin{center}
{\scriptsize

\begin{tabular}{|c|c|c|c|c|c|}
\hline
\multirow{3}{*}{level}	 & state	& \multicolumn{2}{|c|}{\# states in the}	&  \multirow{3}{*}{ } state space	& $\text{PCP}_{0.2}$ \\\cline{3-4}
 & 	 dimensions &original	& pruned	& reduction	&  arms \\
 & 	 & space	& space	& \%	& oracle  \\

\hline
0	 & 10x10x12	 & 153600	 & 1200 & 	 00.00 & ---\\
\hline
1	 & 10x10x24	 & 72968 & 	1140  &  52.50	 & 54 \\
\hline
3	 & 20x20x24	 & 6704	 & 642	&  95.64	 & 51 \\
\hline
5	 & 40x40x24	 & 2682	 & 671	 & 98.25	 & 50 \\
\hline 
7	 & 80x80x24	 & 492	 & 492& 	 99.67	 & 50 \\
\hline
detection pruning	 & 80x80x24	 & 492	 & 492& 	 99.67	 & 44 \\
\hline
\end{tabular} }
\caption{\label{tab:pruning} For each level of the cascade we present
  the reduction of the size of the state space after pruning each
  stage and the quality of the retained hypotheses measured using
  PCP$_{0.2}$. As a baseline, we compare to pruning the same number of
  states in the HoG detection map (see text).}
\end{center}
\end{table}

As shown in Table~\ref{tab:overall}, the cascade performs comparably
with the state-of-the-art on all parts, significantly outperforming earlier
work.  We also compared to a much simpler approach, inspired
by~\citet{pff-cascade} (detector pruning + rich features): We prune by
thresholding each unary detection map individually to obtain the same
number of states as in our final cascade level, and then apply our
final model with rich features on these states. As can be seen in
Figure~\ref{fig:cascade_pruning}, this baseline performs significantly
worse than our method (performing about as well as a standard PS model
as reported in~\citet{sapp2010}).  This makes a strong case for using
max-marginals (e.g., a global image-dependent quantity) for pruning,
as well as learning how to prune safely and efficiently, rather than
using static thresholds on individual part scores as
in~\citet{pff-cascade}.

\out{
Our previous method~\cite{sapp2010} is the only other PS method which
incorporates image information into the pairwise term of the model.
However it is still an exhaustive inference method.  Assuming all
features have been pre-computed, inference in~\cite{sapp2010} takes an
average of 3.2 seconds, whereas inference using the sparse set of
states in the final stage of the cascade takes on average 0.285
seconds---a speedup of 11.2x\footnote{Run on an Intel Xeon E5450
  3.00GHz CPU with an $80\times80\times24$ state space averaged over
  20 trials.~\cite{sapp2010} uses MATLAB's optimized fft function for
  message passing.}.
}
In Table~\ref{tab:pruning}, we evaluate the test time efficiency and accuracy 
of our system after each successive stage of pruning.  In the early stages,
the state space is too coarse for the MAP state sequence of one of the
pruning models to be meaningfully compared to the fine-resolution
groundtruth, so we report PCP scores of the best possible as-yet
unpruned state left in the original space.  We choose a tight
PCP$_{0.2}$ threshold to get an accurate understanding of whether or
not we have lost well-localized limbs.  As seen in
Table~\ref{tab:pruning}, the drop in PCP$_{0.2}$ is small and linear,
whereas the pruning of the state space is exponential---half of the
states are pruned in the first
stage. As a baseline, we evaluate the simple detector-based pruning described
above.  This leads to a significant loss of correct hypotheses, to
which we attribute the poor end-system performance of this baseline
(in Figure~\ref{fig:cascade_pruning} and Table~\ref{tab:overall}), even after
adding richer features.

\begin{table}[t]
\begin{center}
{\scriptsize
\begin{tabular}{|c|c|c|c|c|c|}
\hline
Method &	Torso &	Head	 &	Upper &	Lower	 &	Total \\
 &	&		 &	Arms	 &	Arms	 &	\\
\hline\hline
\textbf{Buffy} & & & & &\\
\hline
\citet{andriluka09}	 &	90.7	 &	95.5	 &	79.3	 &	41.2	 &		73.5 \\
\hline
\citet{eichner09}		 &	98.7 	 &		97.9		 &	82.8	& 59.8	 &	80.1 \\
\hline
\citet{sapp2010}	&	100		 &	100		 &	91.1 &		65.7	 &	85.9 \\
\hline
CPS (ours)	&	99.6 &	98.7	&	91.9	&	64.5	& 	85.2 \\
\hline
Detector pruning &		99.6	&	87.3&		90.0	&	55.3 	&		79.6 \\
\hline
\hline						
\textbf{PASCAL stickmen} & & & & & \\
\hline
\citet{eichner09}		&	97.22	&	88.60	&	73.75	&	41.53		&	69.31 \\
\hline
\citet{sapp2010}	&	100	&	98.0	&	83.9	&	54.0		&	79.0 \\
\hline
CPS (ours)	&	100	&	99.2	&	81.5	&	53.9		&	78.3 \\
\hline
\end{tabular}
}
\caption{\label{tab:overall} Comparison to other methods at
  PCP$_{0.5}$.  See text for details.  We perform comparably to
  state-of-the-art on all parts, improving on upper arms.
 (\textbf{NOTE}: the numbers included here are slightly different from
  the published version---what is seen here exactly matches the
  publicly available reference implementation at
  \url{http://vision.grasp.upenn.edu/video/}).}
\end{center}
\end{table}



\subsection{Loopy Graphs with Ensemble-\SPC}
\label{sec:videopose}

We evaluated Ensemble-\SPC in two experiments. First, we analyzed the
``best-case'' filtering performance of the summed max-marginal
approximation to the true marginals on a synthetic image segmentation
task, assuming the true scoring function $\score{y}$ is available for
inference. Second, we evaluated the real-world accuracy of our
approach on a difficult, real-world human pose dataset (VideoPose). In
both experiments, the max-marginal ensemble outperforms
state-of-the-art baselines.

\subsubsection{Asymptotic Filtering Accuracy on Synthetic Data}

We first evaluated the filtering accuracy of the max-marginal ensemble
on a synthetic 8-class segmentation task. For this experiment, we
removed variability due to parameter estimation and focused our
analysis on accuracy of inference. We compared our approach to Loopy
Belief Propagation (Loopy BP) \citep{pearl1988probabilistic,mceliece1998turbo,murphy1999loopy}, on a $11 \times 11$ two-dimensional grid
MRF.\footnote{We used the UGM Matlab Toolbox by Mark Schmidt for
  the Loopy BP and Gibbs MCMC comparisons, see: \url{http://people.cs.ubc.ca/~schmidtm/Software/UGM.html}.} For the
ensemble, we used 22 unique ``comb'' tree structures to approximate the
full grid model. To generate a
synthetic instance, we generated unary potentials $\omega_i(k)$
uniformly on $[0,1]$ and pairwise potentials log-uniformly:
$\omega_{ij}(k,k') = e^{-v}$, where $v \sim \mcal{U}[-25,25]$ was
sampled independently for every edge and every pair of classes. (Note
that for the ensemble, we normalized unary and edge potentials by
dividing by the number of times that each potential was included in
any model.) It is well known that inference for such grid MRFs is
generally difficult \citep{koller}, and we observed that Loopy BP
failed to converge for at least a few variables on most examples we
generated.

\begin{figure}[t]
\centering
\includegraphics[width=0.98\textwidth]{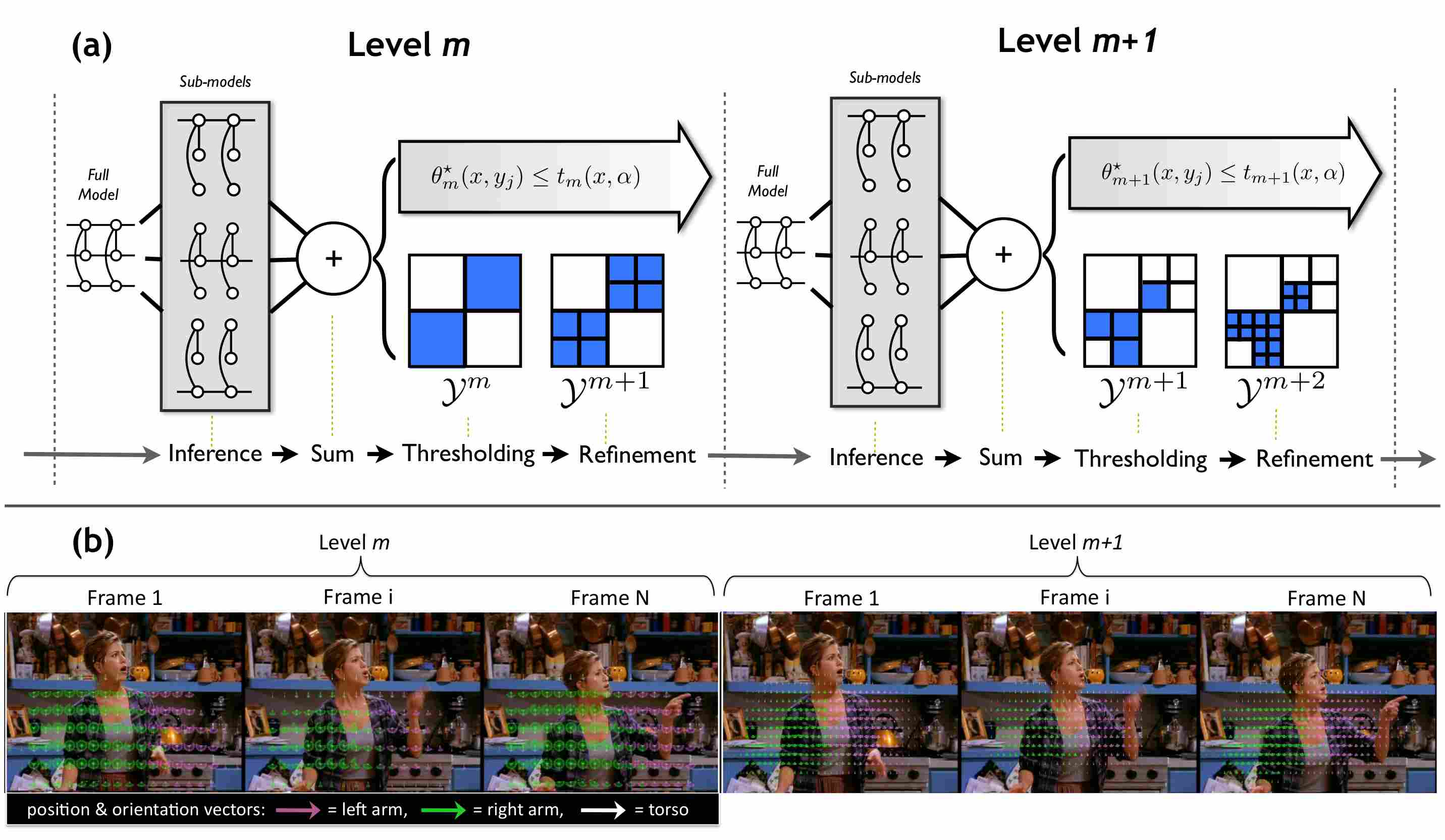}
\caption{\small \label{fig:overview} \textbf{(a)} Schematic overview
  Ensemble-\SPC for human pose tracking. The $m$'th level of the
  cascade takes as input a sparse set of states $\mcal{Y}^m$ for each
  variable $y_j$. The full model is decomposed into constituent
  sub-models (above, the three tree models used in the pose tracking
  experiment) and sparse inference is run. Next, the max marginals of
  the sub-models are summed to produce a single max marginal for each
  variable assignment: $\mmarg{y_j} = \sum_p \pmmarg{y_j}$. Note that
  each level and each constituent model will have different parameters
  as a result of the learning process. Finally, the state spaces are
  thresholded based on the max-marginal scores and low-scoring states
  are filtered. Each state is then refined according to a state
  hierarchy (e.g., spatial resolution, or semantic categories) and
  passed to the next level of the cascade. This process can be
  repeated as many times as desired. In \textbf{(b)}, we illustrate
  two consecutive levels of the ensemble cascade on real data, showing
  the filtered hypotheses left for a single video example.}
\end{figure}

We evaluated our approach on 100 synthetic grid MRF instances. For each
instance, we computed the accuracy of filtering using marginals from
Loopy BP, the ensemble, and each individual sub-model. We determined
error rates by counting the number of times ``ground truth'' was
incorrectly filtered if the top $K$ states were kept for each
variable, where we sampled 1000 ``ground truth'' examples from the
true joint distribution using Gibbs sampling. To obtain a good
estimate of the true marginals, we restarted the chain for each sample
and allowed 1000 iterations of mixing time.  The result is presented
in Figure \ref{fig:synth} for all possible values of $K$ (filter
aggressiveness.) We found that the ensemble outperformed Loopy BP and
the individual sub-models by a significant margin for all $K$.

\begin{figure}[t]
  \centering
  \includegraphics[width=0.7\textwidth]{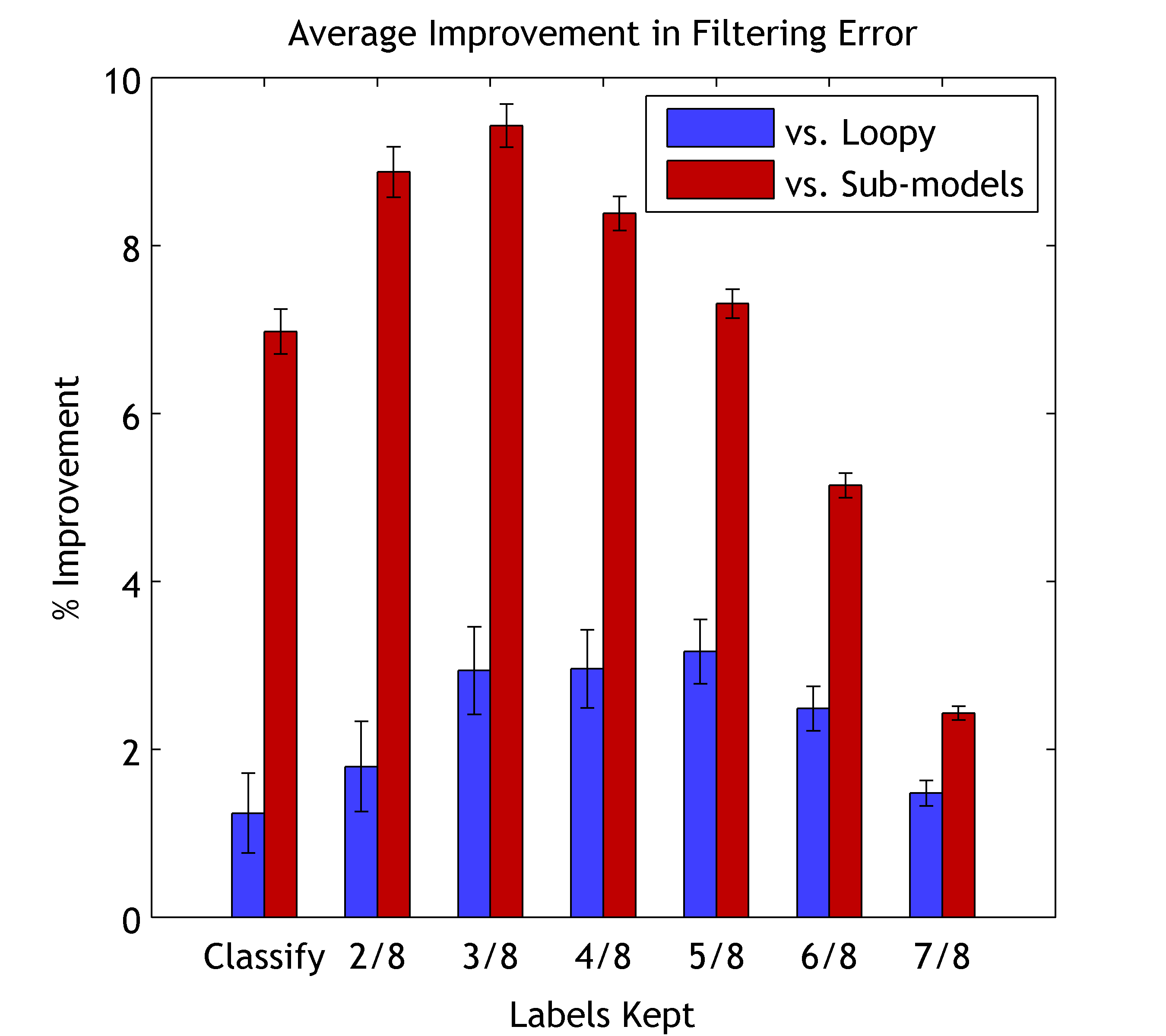}
  \caption{\small Improvement over Loopy BP and constituent
  tree-models on the synthetic segmentation task. Error bars show
  standard error.} 
  \label{fig:synth}  
\end{figure}

We next investigated the question of whether or not the ensembles were
most accurate on variables for which the sub-models tended to
agree. For each variable $y_{ij}$ in each instance, we computed the
mean pairwise Spearman correlation between the ranking of the 8
classes induced by the max marginals of each of the 22 sub-models. We
found that complete agreement between all sub-models never occurred
(the median correlation was 0.38). We found that sub-model agreement
was significantly correlated ($p < 10^{-15}$) with the error of the
ensemble for all values of $K$, peaking at $\rho = -0.143$ at
$K=5$. Thus, increased agreement predicted a decrease in error of the
ensemble. We then asked the question: Does the effect of model
agreement explain the {\em improvement} of the ensemble over Loopy BP?
In fact, the improvement in error compared to Loopy BP was {\em not}
correlated with sub-model agreement for any $K$ (maximum $\rho =
0.0185$, $p < 0.05$). Thus, sub-model agreement does {\em not} explain
the improvement over Loopy BP, indicating that sub-model disagreement
is not related to the difficulty in inference problems that causes
Loopy BP to underperform relative to the ensembles (e.g., due to
convergence failure.)

\out{\todo{Use max-sum loopy BP instead of sum-product.}}

\subsubsection{Articulated Pose Tracking Cascade}

The VideoPose dataset\footnote{The VideoPose dataset is available
  online at {\tt http://vision.grasp.upenn.edu/video/.}} consists of
34 video clips of approximately 50 frames each. The clips were
harvested from three popular TV shows: 3 from {\em Buffy the Vampire
  Slayer}, 27 from {\em Friends}, and 4 from {\em LOST}. Clips were
chosen to highlight a variety of situations and and movements when the
camera is largely focused on a single actor.
In our experiments, we use the {\em Buffy} and half of the {\em
  Friends} clips as training (17 clips), and the remaining {\em
  Friends} and {\em LOST} clips for testing.  In total we test on 901
individual frames. The {\em Friends} are split so no clips from the
same episode are used for both training and testing.  We further set
aside 4 of the {\em Friends} test clips to use as a development
set. Each frame of each clip is hand-annotated with locations of
joints of a full pose model; for simplicity, we use only the torso and
upper arm annotations in this work, as these have the strongest
continuity across frames and strong geometric relationships.

All of the models we evaluated on this dataset share the same basic
structure: a variable for each limb's $(x,y)$ location and angle
rotation (torso, left arm, and right arm) with edges between torso and
arms to model pose geometry. We refer to this basic model, evaluated
independently on each frame, as the ``Single Frame'' approach. For the
VideoPose dataset, we augmented this model by adding edges between
limb states in adjacent frames (Figure \ref{fig:overview}), forming an
intractable, loopy model. Our features in a single frame are the same
as in the beginning levels of the pictorial structure cascade (Section
\ref{sec:pose}): unary features are discretized Histogram of Gradient
(HoG) part detectors scores, and pairwise terms measure relative
displacement in location and angle between neighboring parts.
Pairwise features connecting limbs across time also express geometric
displacement, allowing our model to capture the fact that human limbs
move smoothly over time.

\begin{figure}[t]
  \centering
  \includegraphics[width=0.98\textwidth]{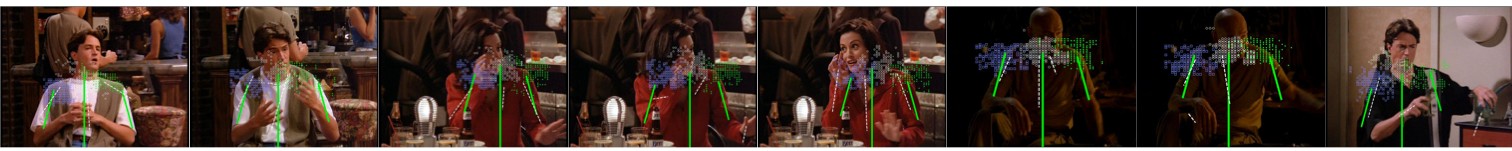}
  \caption{\small \label{fig:results-qualitative} Qualitative test
    results.  Points shown are the position of left/right shoulders and
    torsos at the last level of the ensemble SC (blue square, green dot,
    white circle resp.).  Also shown (green line segments) are the
    best-fitting hypotheses to groundtruth joints, selected from within
    the top 4 max-marginal values. Shown as dotted gray lines is the best
    guess pose returned by the \citep{ferrari08}.}
\end{figure}
    
We learned a coarse-to-fine structured cascade with six levels for
tracking as follows. The six levels use increasingly finer state
spaces for joint locations, discretized into bins of resolution $10
\times 10$ up to $80 \times 80$, with each stage doubling one of the
state space dimensions in the refinement step.  All levels use an
angular discretization of 24 bins.  For the ensemble cascade, we
learned three sub-models simultaneously (Figure \ref{fig:overview}),
with each sub-model accounting for temporal consistency for a
different limb by adding edges connecting the same limb in consecutive
frames.

A summary of results are presented in Figure \ref{fig:tracking}. We
compared the single-frame cascade and the ensemble cascade to a
state-of-the-art single-frame pose detector (Ferrari et
al. \citep{ferrari08}) and to one of the individual sub-models,
modeling torso consistency only (``Torso Only'').  We evaluated the
method from \citep{ferrari08} on only the first half of the test data
due to computation time (taking approximately 7 minutes/frame).  We
found that the ensemble cascade was the most accurate for every joint
in the model, that all cascades outperformed the state-of-the-art
baseline, and, interestingly, that the single-frame cascade
outperformed the torso-only cascade. We suspect that the poor
performance of the torso-only model may arise because propagating only
torso states through time leads to an over-reliance on the relatively
weak torso signal to determine the location of all the limbs. Sample
qualitative output from the ensemble is presented in Figure
\ref{fig:results-qualitative}.

\begin{figure}
  \centering
  \subfigure[\small Decoding Error.]{
    \includegraphics[width=0.48\textwidth]{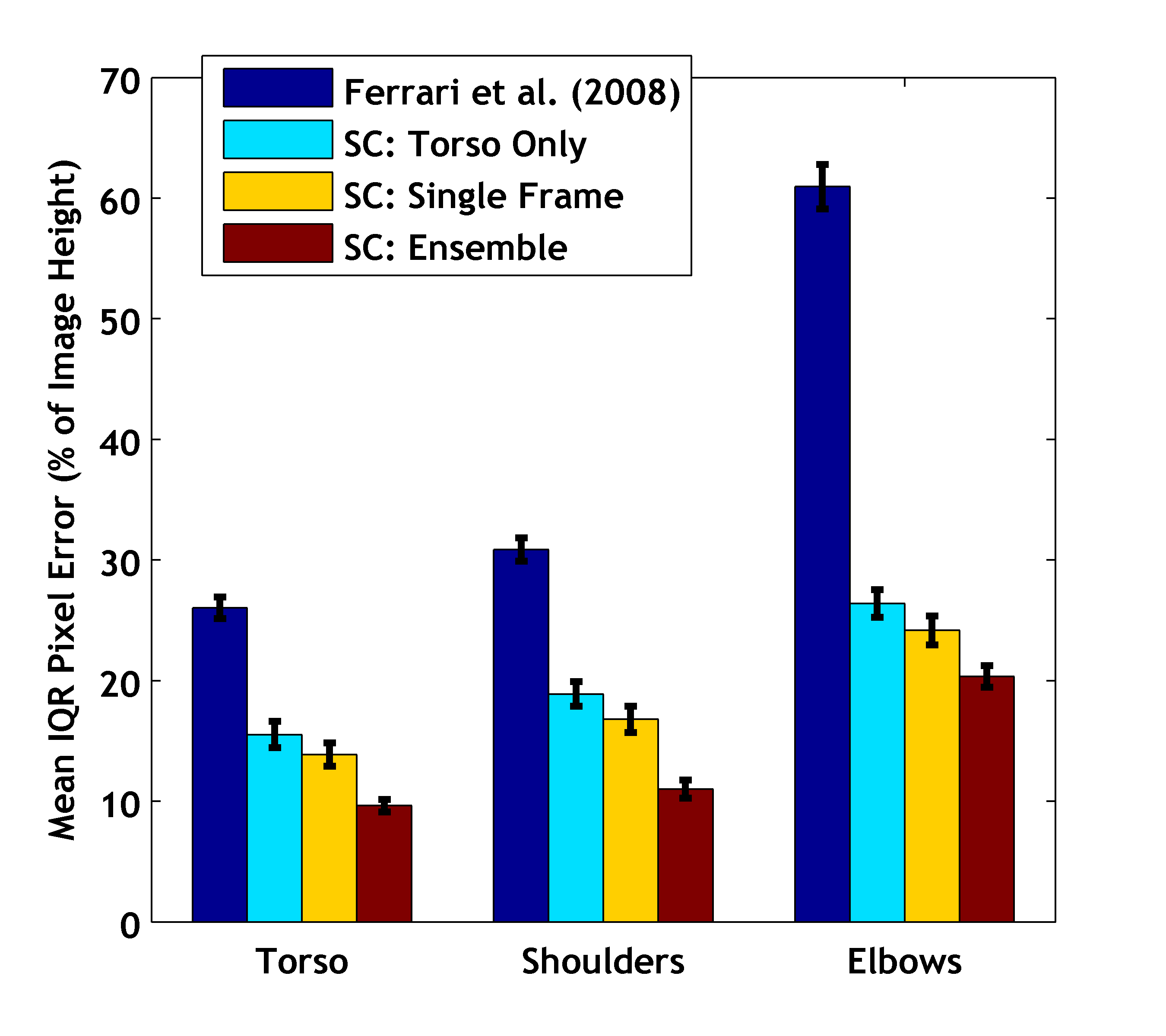}
  }
  \hspace{-0.1in}
  \subfigure[\small Top $K=4$ Error.]{
    \includegraphics[width=0.48\textwidth]{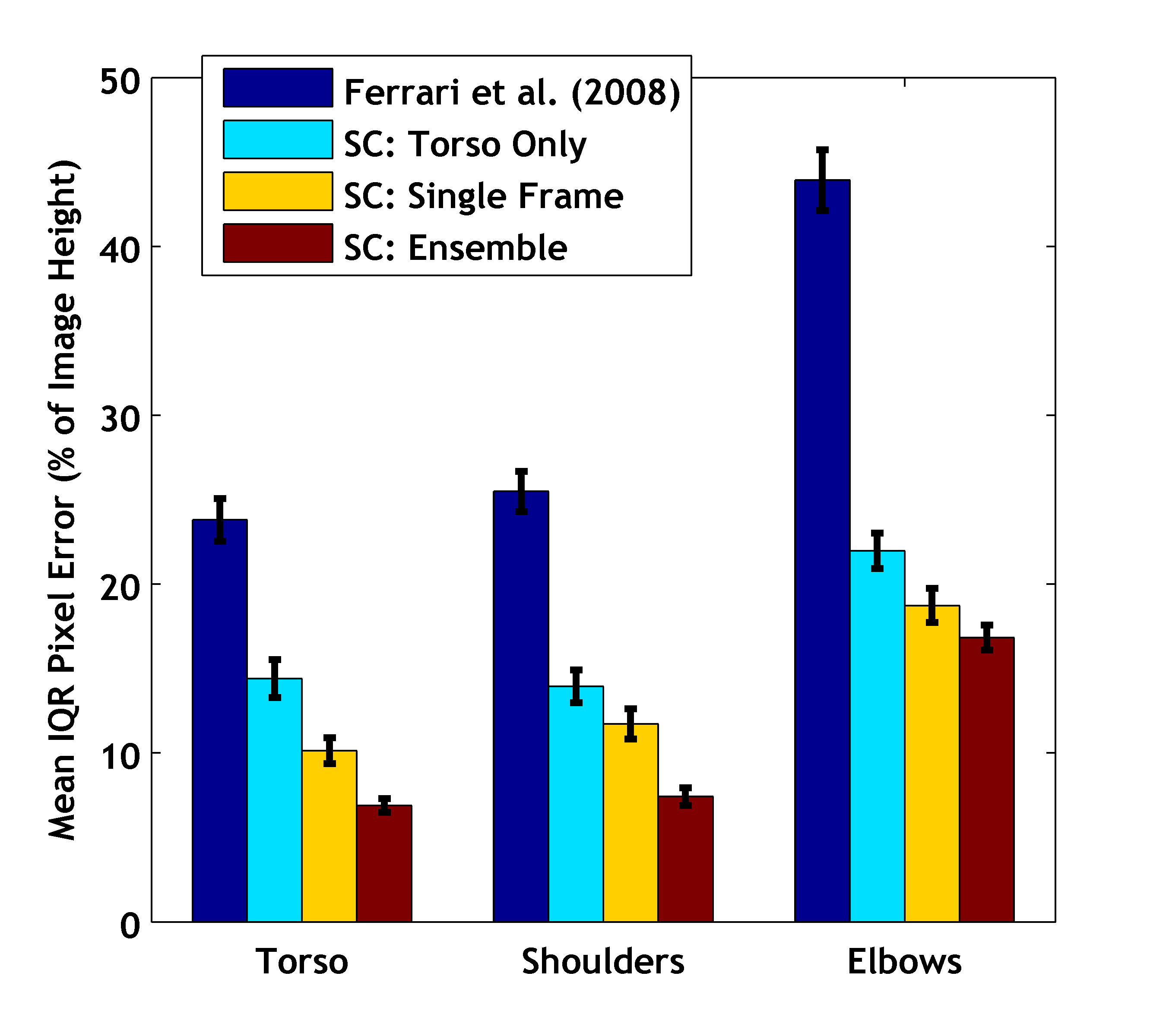}    
  } \\
  \begin{minipage}{1\linewidth}
\centering
    \scalebox{1}{
      \begin{tabular}{cccc}
        \hline \\[-0.5em]
        & {\em State } & $PCP_{0.25}$& {\em Efficiency } \\[0.5em]
        {\em Level} & {\em Dimensions } & {\em in top  K=4} & (\%)  \\[0.5em]
\hline \\[-0.5em]
        0 & $10 \times 10 \times 24$ & -- & -- \\ 
        2 & $20 \times 20 \times 24$ & 98.8& 87.5 \\     
        4 & $40 \times 40 \times 24$ & 93.8& 96.9 \\
        6 & $80 \times 80 \times 24$ & 84.6& 99.2 \\[0.5em]
        \hline
      \end{tabular}
    }
    \vspace{0.1in}
    \begin{center}
      {\small (c) Ensemble efficiency.}
    \end{center}
  \end{minipage}

  \caption{\small \label{fig:tracking} {\bf (a)},{\bf (b)}: Prediction
    error for VideoPose dataset. Reported errors are the average
    distance from a predicted joint location to the true joint for
    frames that lie in the [25,75] inter-quartile range (IQR) of
    errors. Error bars show standard errors computed with respect to
    clips. All SC models outperform \citep{ferrari08}; the ``torso
    only'' persistence cascade introduces additional error compared to
    a single-frame cascade, but adding arm dependencies in the
    ensemble yields the best performance. {\bf (c)}: Summary of test
    set filtering efficiency and accuracy for the ensemble
    cascade. $PCP_{0.25}$ measures Oracle \% of correctly matched limb
    locations given unfiltered states; see \citep{sapp10cascades} for
    more details.}
\end{figure}


\section{Conclusion}

We presented Structured Prediction Cascades, a framework for
adaptively increasing the complexity of structured models on a
per-example basis while maintaining efficiency of inference. 
This allows for the construction and training of
structured models of far greater complexity than was previously
possible. We proposed two novel loss functions, filtering loss and
efficiency loss, that measure the two objectives balanced by the
cascade, and provided generalization bounds for these loss
functions. We proposed a simple sub-gradient based learning algorithm
to minimize these losses, and presented a stage-wise learning
algorithm for the entire cascade in Algorithm \ref{alg:learning}. We
also show how to extend the previous algorithm and theoretical results
to the setting in which exact inference is intractable, using
Ensemble-\SPC. Finally, we showed experimentally state-of-the-art performance 
across multiple domains.



\newpage
\appendix

\section{Proofs of Theorems~\ref{thm:gencascade} and~\ref{thm:gencascade-ensemble} }
\label{sec:proofs}
We first summarize the Rademacher and Gaussian complexity definitions and results from
\citet{bartlettM02} required to prove the theorems. 

\begin{defn}[Rademacher and Gaussian complexities]
Let $H: \cX \mapsto \reals$ be a function class and $x^{1},\ldots,x^{n}$ be $n$ independent 
samples from a fixed distribution. Define the random variables:
\begin{eqnarray}
\hR(H) &=& \bbE_{\sigma}\left[\sup_{h\in H} \left. \left| \frac{2}{n}\sum_{i=1}^{n} \sigma_{i}h(x^{i}) \right| \right|  x^{1},\ldots,x^{n} \right],\\
\hG(H) &=& \bbE_{g}\left[\sup_{h\in H} \left. \left| \frac{2}{n}\sum_{i=1}^{n} g_{i}h(x^{i}) \right| \right|  x^{1},\ldots,x^{n} \right],
\end{eqnarray}
where $\sigma_{i}\in\pm 1$ are independent uniform and $g_{i}\in\reals$ are independent standard Gaussian.  Then $R(H) = \bbE[\hR(H)]$ and $G(H) = \bbE[\hG(H)]$ are
the Rademacher and Gaussian complexities of $H$.
\end{defn}
Consider a general loss function ${\Phi}(y, \bh(x))$ where
 $\bh(x) \in \reals^{m}$ represents the prediction function.  In our case,
 $\bh(x)$ is vector of clique assignment scores  $\theta^{\top}\bft_{c}(x,y_{c})$ of dimension
 $\sum_{c\in\cC}|\cY_{c}|$, indexed by $y_{c}$ (a clique  and its assignment).
 This vector $\bh(x)$ contains all the information needed to compute the max-marginals and threshold for a given example $x$.  Both $\cL_{e}$ and $\cL_{f}$ can be written in this general form, as we detail below.

\begin{defn}[Lipschitz continuity with respect to Euclidean norm]

Let $\phi : \reals^{m} \mapsto \reals$, then $\phi$ is Lipschitz continuous
with constant $L(\phi)$ with respect to Euclidean norm if for any $\bz_{1},\bz_{2}\in\reals^{m}$:
\begin{equation} |\phi(\bz_{1}) -\phi(\bz_{2})| \le L(\phi) ||\bz_{1}-\bz_{2}||_{2}.\end{equation}
\end{defn}
 
We recall the relevant results in the following theorem:
\begin{thm}[Bartlett and Mendelson, 2002]
\label{thm:genbound}
  Consider a loss function $\Phi: \cY\times\reals^{m}\mapsto \reals$ and a dominating cost
  function $\phi: \cY\times\reals^{m}\mapsto \reals$ such that $\Phi(y,\bz) \le \phi(y,\bz)$. Let $H:
  \mathcal{X} \mapsto \reals^{m}$ be a vector-valued class of functions. Then for
  any integer $n$ and any $0 < \delta < 1$, with probability $1 -
  \delta$ over samples of length $n$, every $\bh$ in $H$ satisfies
  \begin{equation}
    \label{eq:genbound}
    \bbE[\Phi(Y,\bh(X))] \le  \hat{\bbE} [\phi(Y,\bh(X))] + 
    R_n(\tilde{\phi} \;\circ\; H) + \sqrt{ \frac{8 \ln(2/\delta)}{n} },
  \end{equation}
  where $\tilde{\phi} \circ H$ is a class of functions defined by centered composition of $\phi$
  with $\bh \in H$, $\tilde{\phi} \circ \bh = \phi(y,\bh(x)) - \phi(y,0)$.

  Furthermore, Rademacher complexity can be bounded using Gaussian complexity: there are absolute constants $c$ and $C$ such that for
  every class $H$ and every integer $n$,
  \begin{equation}
    \label{eq:rad_to_gaus}
    cR_n(H) \le G_n(H) \le (C \ln  n) R_n(H).
  \end{equation}
  Let  $H: \cX \to \reals^{m} $ be a class of functions
  that is the direct sum of real-valued classes $H_1, \dots,
  H_m$. Then, for every integer $n$ and every sample $(x^{1}, \dots,x^{n})$,
  \begin{equation}
    \label{eq:multi_gaus}
    \hG_n(\phi \circ H) \le 2L(\phi) \sum_{i=1}^m \hG_n(H_i),
  \end{equation}
  where $L(\phi)$ is the Lipschitz constant of $\phi$ with respect to Euclidean distance.  
 Finally, for the 2-norm-bounded linear class of functions,  $H = \{x \mapsto \theta^\top\bft(x) \mid ||\theta||_2 \le B,
  ||\bft(x)||_2 \le 1\}$,
  \begin{equation}
    \label{eq:lingaus}
    \hat{G}_n(H) \le \frac{2B}{\sqrt{n}}.
  \end{equation}
\end{thm}

\subsection{Proof of Theorem~\ref{thm:gencascade}}
We will express our loss functions $\cL_{e}$ and $\cL_{f}$
and dominating loss functions $\cL_{e}^{\gamma}$ and $\cL_{f}^{\gamma}$, in the
in terms of the framework above.  We reproduce the definitions side-by-side in a slightly modified form below, where $m=\sum_{c\in\cC}|\cY_{c}|$ and the $\gamma$-margin step-function dominates the step-function $r_{\gamma}(z) \ge \ind{z\le 0}$ by construction:
 \begin{eqnarray}
    \cL_f(x,y;\theta,\alpha) &=& \ind{\score{y}- \thr \le 0},\\
     \cL_f^{\gamma}(x,y;\theta,\alpha) &=& r_{\gamma}(\score{y}- \thr),\\
    \cL_e(x,y;\theta,\alpha) &=&  \frac{1}{m} \sum_{c\in C, y_c \in \cY_c} \ind{\thr-\mmarg{y_c}\le 0},\\
    \cL_e^{\gamma}(x,y;\theta,\alpha) &=&  \frac{1}{m} \sum_{c\in C, y_c \in \cY_c} r_{\gamma}( \thr-\mmarg{y_c}).
 \end{eqnarray}

We ``vectorize'' our scoring function $\theta$ and assignments $y$  by defining vector-valued functions, where the vectors are indexed by clique assignments, $y_{c}$, with total dimension 
$m$.

\begin{defn}[Vectorization]
 \begin{eqnarray}
 \bh_{y_{c}} (x) &\triangleq& \theta^{\top}\bft_{c}(x,y_{c})\\
 \bv_{y_{c}}(y') &\triangleq& \ind{y'_{c}=y_{c}} \\
  \score{y} &=& \bh(x)^{\top}\bv(y)
 \end{eqnarray}
\end{defn}
Clearly, the $m$-dimensional vector $\bh(x)$ contains all the information needed to compute the max-marginals and threshold for a given example $x$ (we assume $\alpha$ is fixed).  Hence we can define the losses in the form of Theorem~\ref{thm:genbound}:
 \begin{eqnarray}
    \Phi_{f}(y,\bh(x)) &=&  \cL_f(x,y;\theta,\alpha)\\ 
    \phi_{f}(y,\bh(x)) &=&  \cL_f^{\gamma}(x,y;\theta,\alpha) \\
    \Phi_{e}(y,\bh(x)) &=&  \cL_e(x,y;\theta,\alpha) \\   
    \phi_{e}(y,\bh(x)) &=&     \cL_e^{\gamma}(x,y;\theta,\alpha) 
 \end{eqnarray}
What remains is to calculate the Lipschitz constants of $\phi_{f}$ and  $\phi_{e}$.

\begin{thm}
  \label{thm:lipschitz}
  $\phi_f(y,\cdot)$ and $\phi_e(y,\cdot)$ are Lipschitz (with respect
  to Euclidean distance on $\reals^m$) with constant
  $\sqrt{2\ncliques|}/\gamma$ for all $y \in \cY$.
\end{thm}


To prove Theorem \ref{thm:lipschitz}, we bound Lipschitz constants of constituent
functions of  $\phi_{f}$ and $\phi_{e}$.

\begin{lem}
  \label{lem:trick}
  Fix any $y \in \cY$ and let $\phi_{1} : \reals^{m}\mapsto \reals$ be defined as 
  $$\phi_{1}(\bz) = \bz^{\top}\bv(y) - \max_{y'\in\cY} \bz^{\top}\bv(y').$$ 
  Then $\phi_{1}(\bz_{1}) - \phi_{1}(\bz_{2}) \le \sqrt{2\ncliques} ||\bz_{1} -\bz_{2}||_2$ for any $\bz_{1},\bz_{2}\in\reals^{m}$.
\end{lem}
\begin{proof}
    For brevity of notation in the proof below, we define 
    $\bv = \bv(y), \bv_{1} = \bv( \argmax_{y'} \bz_{1}^{\top}\bv(y'))$ and 
    $\bv_{2} = \bv( \argmax_{y'} \bz_{2}^{\top}\bv(y'))$, with ties broken arbitrarily but deterministically.  Then,
  \begin{align*}
    \phi_{1}(\bz_{1}) - \phi_{1}(\bz_{2}) & = \bz_{1}^{\top}\bv - \bz_{1}^{\top}\bv_{1} - 
    \bz_{2}^{\top}\bv + \bz_{2}^{\top}\bv_{2}\\
    & = (\bz_{2}-\bz_{1})^{\top}(\bv_{2} - \bv) + \bz_{1}^{\top}(\bv_{2} - \bv_{1})\\
    & \le (\bz_{2}-\bz_{1})^{\top}(\bv_{2} - \bv)  \\
    & \le ||\bz_{2}-\bz_{1}||_{2} ||\bv_{2} - \bv||_{2} \\
    & \le \sqrt{2\ncliques} ||\bz_{1}-\bz_{2}||_2.
  \end{align*}
  The last three steps follow (1) from  the fact that $\bv_{1}$ maximizes
  $\bz_{1}^{\top}\bv(y')$ (so that $\bz_{1}^{\top} (\bv_{2} - \bv_{1})$ is negative), (2) from
  the Cauchy-Schwarz inequality, and (3) from the fact that there are $\ncliques$
  cliques, each of which can contribute at most a single non-zero entry
  in $\bv$ or $\bv_{2}$.
\end{proof}

\begin{lem}
  \label{lem:mmargdiff}
Fix any $y \in \cY$ and let $\phi_{2} : \reals^{m}\mapsto \reals$ be defined as 
$$\phi_{2}(\bz) = \bz^{\top}\bv(y) - \frac{1}{m} \sum_{c\in\cC,y'_{c}\in\cY_{c}} \max_{y'':y''_{c}=y'_{c}} \bz^{\top}\bv(y'').$$ 
  Then $\phi_{2}(\bz_{1}) - \phi_{2}(\bz_{2}) \le \sqrt{2\ncliques} ||\bz_{1} -\bz_{2}||_2$ for any $\bz_{1},\bz_{2}\in\reals^{m}$.
\end{lem}

\begin{proof}
  Let $\bv = \bv(y),  \bv_{1y'_{c}} = \bv(\argmax_{y'': y''=y'_{c}} \bz_{1}^{\top}\bv(y''))$ and
  $\bv_{2y'_{c}} = \bv(\argmax_{y'': y''=y'_{c}} \bz_{2}^{\top}\bv(y''))$.
  \begin{align*}
    \phi_{2}(\bz_{1})-\phi_{2}(\bz_{2}) & = \frac{1}{m} \sum_{c\in\cC,y'_{c}\in\cY_{c}} 
    \bz_{1}^{\top}\bv - \bz_{1}^{\top}\bv_{1y'_{c}} - 
    \bz_{2}^{\top}\bv + \bz_{2}^{\top}\bv_{2y'_{c}}\\
   & = \frac{1}{m} \sum_{c\in\cC,y'_{c}\in\cY_{c}}  (\bz_{2}-\bz_{1})^{\top}(\bv_{2y'_{c}} - \bv) + \bz_{1}^{\top}(\bv_{2y'_{c}} - \bv_{1y'_{c}})\\
    & \le   \frac{1}{m} \sum_{c\in\cC,y'_{c}\in\cY_{c}}  (\bz_{2}-\bz_{1})^{\top}(\bv_{2y'_{c}} - \bv) \\
    & \le \frac{1}{m}   \sum_{c\in\cC,y'_{c}\in\cY_{c}}  \sqrt{2\ncliques} ||\bz_{1}-\bz_{2}||_2 =  \sqrt{2\ncliques} ||\bz_{1}-\bz_{2}||_2 .
  \end{align*}
The inequalities follow using a similar argument to previous lemma, but made separately for each $y'_{c}$.
\end{proof}

\begin{lem}
  \label{lem:phidiff}
Fix any $y \in \cY$ and let  
$$\phi_{3}(\bz) = \alpha\phi_{1}(\bz) + (1-\alpha)\phi_{2}(\bz) = 
 \overbrace{\bz^{\top}\bv(y)}^{\score{y}} - \overbrace{\left(\alpha \max_{y'} \bz^{\top}\bv(y') + \frac{1-\alpha}{m} \sum_{y'_{c}} \max_{y'':y''_{c}=y'_{c}} \bz^{\top}\bv(y'')\right)}^{\thr},$$
where the over-braces show the relationship to the score of the correct label sequence and the threshold, assuming $\bz = \bh(x)$.  
Then $\phi_{3}(\bz_{1}) - \phi_{3}(\bz_{2}) \le  \sqrt{2\ncliques}||\bz_{1}-\bz_{2}||_2$ for any $\bz_{1},\bz_{2}\in\reals^{m}$
and the Lipschitz constant of $\phi_{f}= r_{\gamma} \circ \phi_{3}$ is bounded by $ \sqrt{2\ncliques}/\gamma$.
\end{lem}
\begin{proof}
 Combining two previous  lemmas we have that 
  $$\phi_{3}(\bz_{1}) - \phi_{3}(\bz_{2}) = \alpha(\phi_{1}(\bz_{1}) - \phi_{1}(\bz_{2})) +  (1-\alpha)(\phi_{2}(\bz_{1}) - \phi_{2}(\bz_{2}) ) \le \sqrt{2\ncliques}||\bz_{1}-\bz_{2}||_2.$$
To show that $\phi_f$  is Lipschitz
continuous with constant $\sqrt{2\ncliques}/\gamma$, we note that
$\phi_{f} = r_{\gamma} \circ \phi_{3}$ so  $L(\phi_f) = L(r_\gamma)\cdot L(\phi_{3}) \le
\sqrt{2\ncliques}/\gamma.$ 
\end{proof}

Next, we show that $\phi_e$ is Lipschitz
continuous with the same constant. 
\begin{lem}
Fix $c \in \cC$ and  $y_{c}\in\cY$  and let $\phi_{[y_{c}]} : \reals^{m}\mapsto \reals$ be defined as 
$$\phi_{{[y_{c}]}}(\bz) =  \overbrace{\left(\max_{y' : y'_{c} = y_{c}} \bz^{\top}\bv(y' )\right)}^{\mmarg{y_{c}}} - \overbrace{\left(\alpha \max_{y'} \bz^{\top}\bv(y') + \frac{1-\alpha}{m} \sum_{y'_{c'}} \max_{y'':y''_{c'}=y'_{c'}} \bz^{\top}\bv(y'')\right)}^{\thr},$$ 
where the over-braces show the relationship to max-marginal of $y_{c}$ and 
the threshold and, assuming $\bz = \bh(x)$.
  Then $\phi_{{[y_{c}]}}(\bz_{1}) - \phi_{[y_{c}]}(\bz_{2}) \le \sqrt{2\ncliques} ||\bz_{1} -\bz_{2}||_2$ for any $\bz_{1},\bz_{2}\in\reals^{m}$.
\end{lem}
\begin{proof}
  We once again apply the trick from the proof of Lemma~\ref{lem:trick}.  Let 
  \begin{eqnarray*}
  \bv_{1} &=& \bv( \argmax_{y'} \bz_{1}^{\top}\bv(y')),\;\;\;\;\;\;\;\;\;
  \bv_{2} = \bv( \argmax_{y'} \bz_{2}^{\top}\bv(y')),\\
  \bv_{1y'_{c'}} &=& \bv(\argmax_{y'': y''=y'_{c'}} \bz_{1}^{\top}\bv(y'')),\;\;\;\;
  \bv_{2y'_{c'}} = \bv(\argmax_{y'': y''=y'_{c'}} \bz_{2}^{\top}\bv(y'')).
  \end{eqnarray*}
Then, we have:
  \begin{align*} 
    \phi_{{[y_{c}]}}(\bz_{1}) - \phi_{{[y_{c}]}}(\bz_{2})  
    = \; & 
    (\bz_{1}^{\top}\bv_{1y_{c}}  - \bz_{2}^{\top}\bv_{2y_{c}}) + 
    \alpha (\bz_{2}^{\top}\bv_{2} - \bz_{1}^{\top}\bv_{1}) +
    \frac{1-\alpha}{m}\sum_{y'_{c'}} 
    (\bz_{2}^{\top}\bv_{2y'_{c'}}  - \bz_{1}^{\top}\bv_{1y'_{c'}}) \\
    \le & \; 
    (\bz_{1}-\bz_{2})^{\top}\bv_{1y_{c}}  + 
    \alpha  (\bz_{2}-\bz_{1})^{\top} \bv_{2} +
    \frac{1-\alpha}{m}\sum_{y'_{c'}} 
    (\bz_{2}-\bz_{1})^{\top}\bv_{2y'_{c'}} \\
    = &\;       \frac{1}{m}\sum_{y'_{c'}} 
(\bz_{1}-\bz_{2})^{\top}  \left(\bv_{1y_{c}}  - \alpha \bv_{2} -  (1-\alpha)\bv_{2y'_{c'}} \right)\\    
    \le & \; \frac{1}{m}\sum_j \sqrt{2\ncliques}||\bz_{1}-\bz_{2}||_2  = \sqrt{2\ncliques}||\bz_{1}-\bz_{2} ||_2.
  \end{align*}
Here once again we have condensed the argument similar to Lemma \ref{lem:trick}.

\end{proof}
Finally, we note that $\phi_e(\bz) = 1/m\sum_i
r_\gamma(\phi_{[y_{c}]}(\bz))$. Therefore $L(\phi_e) = 1/m \sum_i
\sqrt{2\ncliques}/\gamma = \sqrt{2\ncliques}/\gamma$,
 thus completing the proof of Theorem \ref{thm:lipschitz}.
Now turning back to Theorem~\ref{thm:gencascade},  
we note that the class of functions $H$ we are working with is 
the direct sum of $m$ linear classes each bounded by norm $B$.
Hence we complete the proof of Theorem~\ref{thm:gencascade}, by using Theorem~\ref{thm:genbound}, with $R_{n}(\tphi_{f}\circ H) = R_{n}(\tphi_{e}\circ H) \le \frac{cmB\sqrt{\ncliques}}{\gamma\sqrt{n}}$ for some constant $c$.

\subsection{Proof of Theorem \ref{thm:gencascade-ensemble}}

We define  
\begin{eqnarray*} 
\Phi_{joint}(y,\bh(x))  &\triangleq &  \cL_{joint}(x,y;\theta,\alpha) = \ind{\left(\sum_{p} \pscore{y}- \pthr\right) \le 0}, \\
    \phi_{joint}(y,\bh(x)) & \triangleq & \cL_{joint}^{\gamma}(x,y;\theta,\alpha) = r_{\gamma}\left(\sum_{p} \pscore{y}- \pthr\right)
\end{eqnarray*} 
     
Once again, we fix any $y \in \cY$ and for each $p$, let  (similar to Lemma~\ref{lem:phidiff})
$$\phi_{3}(\bz_{p}) =  
 \overbrace{\bz_{p}^{\top}\bv(y)}^{\pscore{y}} - \overbrace{\left(\alpha \max_{y'} \bz_{p}^{\top}\bv(y') + \frac{1-\alpha}{m} \sum_{y'_{c}} \max_{y'':y''_{c}=y'_{c}} \bz_{p}^{\top}\bv(y'')\right)}^{\pthr},$$
where the over-braces show the relationship to the score of the correct label sequence under model $p$ and the threshold for model $p$, assuming $\bz_{p} = \bh_{p}(x)$ of model $p$.  

Then $\phi_{joint}(y,\bh(x)) = r_\gamma(\sum_p \phi_{3}(\sum_{p} \bz_{p}))$ has Lipschitz constant  $\sqrt{2\ncliques}P/\gamma$, since we can apply Lemma~\ref{lem:phidiff} for each $p$, and $\phi_{joint}(y,\bh(x)) = r_\gamma\left(\sum_p \pscore{y} -\pthr \right)$ has Lipschitz constant at most $\sqrt{2\ncliques}P/\gamma$ because if composition with $r_{\gamma}$ and the sum of $P$ identical terms.  In  Theorem~\ref{thm:gencascade-ensemble}, our function
class $H$ is the direct sum of $m*P$ linear classes each bounded by norm $B/P$, hence
 $R_{n}(\tphi_{joint}\circ H) \le \frac{cmPB\sqrt{\ncliques}}{\gamma\sqrt{n}}$ for some constant $c$.

\label{app:theorem}







\vskip 0.2in
\bibliography{../shared/refs}

\begin{thebibliography}{45}
\providecommand{\natexlab}[1]{#1}
\providecommand{\url}[1]{\texttt{#1}}
\expandafter\ifx\csname urlstyle\endcsname\relax
  \providecommand{\doi}[1]{doi: #1}\else
  \providecommand{\doi}{doi: \begingroup \urlstyle{rm}\Url}\fi

\bibitem[Agarwal et~al.(2011)Agarwal, Duchi, Bartlett, and Levrard]{agarwal11}
A.~Agarwal, J.~Duchi, P.~Bartlett, and C.~Levrard.
\newblock Oracle ineqaulities for computationally budgeted model selection.
\newblock In \emph{Proc. COLT}, 2011.

\bibitem[Akaike(1974)]{aic74}
H.~Akaike.
\newblock A new look at the statistical model identification.
\newblock \emph{Automatic Control, IEEE Transactions on}, 19\penalty0
  (6):\penalty0 716 -- 723, dec 1974.

\bibitem[Andriluka et~al.(2009)Andriluka, Roth, and Schiele]{andriluka09}
M.~Andriluka, S.~Roth, and B.~Schiele.
\newblock {Pictorial structures revisited: People detection and articulated
  pose estimation}.
\newblock In \emph{Proc. CVPR}, 2009.

\bibitem[Barron et~al.(1999)Barron, Birg{\'e}, and Massart]{barron1999risk}
A.~Barron, L.~Birg{\'e}, and P.~Massart.
\newblock Risk bounds for model selection via penalization.
\newblock \emph{Probability theory and related fields}, 113\penalty0
  (3):\penalty0 301--413, 1999.

\bibitem[Bartlett and Mendelson(2002)]{bartlettM02}
P.~Bartlett and S.~Mendelson.
\newblock Rademacher and {G}aussian complexities: Risk bounds and structural
  results.
\newblock \emph{JMLR}, 3:\penalty0 463--482, 2002.

\bibitem[Bartlett et~al.(2002)Bartlett, Boucheron, and
  Lugosi]{bartlett02selection}
P.~L. Bartlett, S.~Boucheron, and G.~Lugosi.
\newblock Model selection and error estimation.
\newblock \emph{Machine Learning}, 48:\penalty0 85--113, 2002.

\bibitem[Bertsekas(1999)]{bertsekas99}
D.~Bertsekas.
\newblock \emph{Nonlinear Programming}.
\newblock Athena Scientific, second edition, 1999.

\bibitem[Bottou and Bousquet(2008)]{bottou-bousquet-2008}
L.~Bottou and O.~Bousquet.
\newblock The tradeoffs of large scale learning.
\newblock In \emph{Neural Information Processing Systems}, pages 161--168,
  2008.

\bibitem[Carreras. et~al.(2008)Carreras., Collins, and Koo]{carreras2008tag}
X.~Carreras., M.~Collins, and T.~Koo.
\newblock {TAG, dynamic programming, and the perceptron for efficient,
  feature-rich parsing}.
\newblock In \emph{Proc. CoNLL}, 2008.

\bibitem[Cavallanti et~al.(2007)Cavallanti, Cesa-Bianchi, and
  Gentile]{cavallanti2007tbh}
G.~Cavallanti, N.~Cesa-Bianchi, and C.~Gentile.
\newblock {Tracking the best hyperplane with a simple budget Perceptron}.
\newblock \emph{Machine Learning}, 69\penalty0 (2):\penalty0 143--167, 2007.

\bibitem[Charniak(2000)]{charniak2000maximum}
E.~Charniak.
\newblock A maximum-entropy-inspired parser.
\newblock In \emph{Proc. NAACL}, 2000.

\bibitem[Chiang et~al.(2005)Chiang, Lopez, Madnani, Monz, Resnik, and
  Subotin]{hiero}
D.~Chiang, A.~Lopez, N.~Madnani, C.~Monz, P.~Resnik, and M.~Subotin.
\newblock The hiero machine translation system: extensions, evaluation, and
  analysis.
\newblock In \emph{Proc. HLT/EMNLP}, pages 779--786, Stroudsburg, PA, USA,
  2005. Association for Computational Linguistics.

\bibitem[Crammer et~al.(2003)Crammer, Kandola, and
  Singer]{Crammer03onlineclassification}
K.~Crammer, J.~Kandola, and Y.~Singer.
\newblock Online classification on a budget.
\newblock In \emph{Proc. NIPS}. MIT Press, 2003.

\bibitem[Daum{\'e} et~al.(2009)Daum{\'e}, Langford, and Marcu]{searn}
H.~Daum{\'e}, J.~Langford, and D.~Marcu.
\newblock Search-based structured prediction.
\newblock \emph{Machine Learning}, 75\penalty0 (3):\penalty0 297--325, 2009.
\newblock URL \url{http://dx.doi.org/10.1007/s10994-009-5106-x}.

\bibitem[Dekel et~al.(2008)Dekel, Shalev-Shwartz, and
  Singer]{dekel2008forgetron}
O.~Dekel, S.~Shalev-Shwartz, and Y.~Singer.
\newblock {The Forgetron: A kernel-based Perceptron on a budget}.
\newblock \emph{SIAM Journal on Computing}, 37\penalty0 (5):\penalty0
  1342--1372, 2008.

\bibitem[Devroye et~al.(1996)Devroye, Gy{\"o}rfi, and
  Lugosi]{devroye1996probabilistic}
L.~Devroye, L.~Gy{\"o}rfi, and G.~Lugosi.
\newblock \emph{A probabilistic theory of pattern recognition}, volume~31.
\newblock Springer Verlag, 1996.

\bibitem[Eichner. and Ferrari(2009)]{eichner09}
M.~Eichner. and V.~Ferrari.
\newblock {Better appearance models for pictorial structures}.
\newblock In \emph{Proc. BMVC}, 2009.

\bibitem[Felzenszwalb and Huttenlocher(2005)]{felzps}
P.~Felzenszwalb and D.~Huttenlocher.
\newblock {Pictorial structures for object recognition}.
\newblock \emph{IJCV}, 61\penalty0 (1):\penalty0 55--79, 2005.

\bibitem[Felzenszwalb et~al.(2010)Felzenszwalb, Girshick, and
  McAllester]{pff-cascade}
P.~Felzenszwalb, R.~Girshick, and D.~McAllester.
\newblock Cascade object detection with deformable part models.
\newblock In \emph{Proc. CVPR}, 2010.

\bibitem[Ferrari et~al.(2008)Ferrari, Marin-Jimenez, and Zisserman]{ferrari08}
V.~Ferrari, M.~Marin-Jimenez, and A.~Zisserman.
\newblock {Progressive search space reduction for human pose estimation}.
\newblock In \emph{Proc. CVPR}, 2008.

\bibitem[Fleuret and Geman(2001)]{geman2001}
F.~Fleuret and D.~Geman.
\newblock Coarse-to-fine face detection.
\newblock \emph{IJCV}, 41\penalty0 (1/2), 2001.

\bibitem[Gao and Koller(2011)]{GaoK11}
T.~Gao and D.~Koller.
\newblock Active classification based on value of classifier.
\newblock In \emph{Proc. NIPS}, 2011.

\bibitem[Kassel(1995)]{Kassel_Thesis}
R.~Kassel.
\newblock \emph{A Comparison of Approaches to On-line Handwritten Character
  Recognition}.
\newblock PhD thesis, Massachusetts Institute of Technology, 1995.

\bibitem[Koller and Friedman(2009)]{koller}
D.~Koller and N.~Friedman.
\newblock \emph{Probabilistic Graphical Models: Principles and Techniques}.
\newblock The MIT Press, 2009.

\bibitem[Komodakis et~al.(2007)Komodakis, Paragios, and
  Tziritas]{komodakis2007dualdecomp}
N.~Komodakis, N.~Paragios, and G.~Tziritas.
\newblock {MRF} optimization via dual decomposition: Message-passing revisited.
\newblock In \emph{Proc. ICCV}, 2007.

\bibitem[Lafferty et~al.(2001)Lafferty, McCallum, and Pereira]{lafferty01crf}
J.~Lafferty, A.~McCallum, and F.~Pereira.
\newblock Conditional random fields: Probabilistic models for segmenting and
  labeling sequence data.
\newblock In \emph{Proc. ICML}, 2001.

\bibitem[Lefakis and Fleuret(2010)]{Lefakis10}
L.~Lefakis and F.~Fleuret.
\newblock Joint cascade optimization using {A} product of boosted classifiers.
\newblock In \emph{Proc. NIPS}, 2010.

\bibitem[Mallows(1973)]{mallows73}
C.~L. Mallows.
\newblock Some comments on {$C_{P}$}.
\newblock \emph{Technometrics}, 15\penalty0 (4):\penalty0 pp. 661--675, 1973.

\bibitem[McEliece et~al.(1998)McEliece, MacKay, and Cheng]{mceliece1998turbo}
R.J. McEliece, D.J.C. MacKay, and J.F. Cheng.
\newblock Turbo decoding as an instance of {P}earl's Òbelief propagationÓ
  algorithm.
\newblock \emph{J. on Selected Areas in Communications}, 16\penalty0
  (2):\penalty0 140--152, 1998.

\bibitem[Murphy et~al.(1999)Murphy, Weiss, and Jordan]{murphy1999loopy}
K.P. Murphy, Y.~Weiss, and M.I. Jordan.
\newblock Loopy belief propagation for approximate inference: An empirical
  study.
\newblock In \emph{Proc. UAI}, pages 467--475, 1999.

\bibitem[Pearl(1988)]{pearl1988probabilistic}
J.~Pearl.
\newblock \emph{Probabilistic reasoning in intelligent systems: networks of
  plausible inference}.
\newblock Morgan Kaufmann, 1988.

\bibitem[Petrov(2009)]{petrov:PhD}
S.~Petrov.
\newblock \emph{Coarse-to-Fine Natural Language Processing}.
\newblock PhD thesis, University of California at Bekeley, 2009.

\bibitem[Petrov et~al.(2008)Petrov, Haghighi, and
  Klein]{petrov-haghighi-klein:2008:EMNLP}
S.~Petrov, A.~Haghighi, and D.~Klein.
\newblock Coarse-to-fine syntactic machine translation using language
  projections.
\newblock In \emph{Proc. EMNLP}, pages 108--116, 2008.

\bibitem[Ramanan and Sminchisescu(2006)]{deva2006}
D.~Ramanan and C.~Sminchisescu.
\newblock Training deformable models for localization.
\newblock In \emph{CVPR}, 2006.

\bibitem[Rush and Petrov(2012)]{rush12}
A.~Rush and S.~Petrov.
\newblock Vine pruning for efficient multi-pass dependency parsing.
\newblock In \emph{Proc. NAACL}, 2012.

\bibitem[Sapp et~al.(2010{\natexlab{a}})Sapp, Jordan, and Taskar]{sapp2010}
B.~Sapp, C.~Jordan, and B.~Taskar.
\newblock Adaptive pose priors for pictorial structures.
\newblock In \emph{Proc. CVPR}, 2010{\natexlab{a}}.

\bibitem[Sapp et~al.(2010{\natexlab{b}})Sapp, Toshev, and
  Taskar]{sapp10cascades}
B.~Sapp, A.~Toshev, and B.~Taskar.
\newblock Cascaded models for articulated pose estimation.
\newblock In \emph{Proc. ECCV}, 2010{\natexlab{b}}.

\bibitem[Shalev-Shwartz and Srebro(2008)]{shalevshwartz2008soi}
S.~Shalev-Shwartz and N.~Srebro.
\newblock {SVM optimization: inverse dependence on training set size}.
\newblock In \emph{International Conference on Machine learning}, pages
  928--935, 2008.

\bibitem[Shalev-Shwartz et~al.(2007)Shalev-Shwartz, Singer, and
  Srebro]{shalev07}
S.~Shalev-Shwartz, Y.~Singer, and N.~Srebro.
\newblock Pegasos: Primal estimated sub-gradient {SO}lver for {SVM}.
\newblock In \emph{Proc. ICML}, 2007.

\bibitem[Taskar et~al.(2003)Taskar, Guestrin, and Koller]{taskar03max}
B.~Taskar, C.~Guestrin, and D.~Koller.
\newblock Max margin {M}arkov networks.
\newblock In \emph{Proc. NIPS}, 2003.

\bibitem[Vapnik and Chervonenkis(1974)]{vapnik1974theory}
V.~Vapnik and A.~Chervonenkis.
\newblock \emph{Theory of pattern recognition}.
\newblock Nauka, 1974.

\bibitem[Venugopal et~al.(2007)Venugopal, Zollmann, and
  Stephan]{venugopal2007efficient}
A.~Venugopal, A.~Zollmann, and V.~Stephan.
\newblock An efficient two-pass approach to synchronous-{CFG} driven
  statistical {MT}.
\newblock In \emph{Proc. HLT-NAACL}, pages 500--507, 2007.

\bibitem[Viola and Jones(2001)]{viola02}
P.~Viola and M.~Jones.
\newblock Rapid object detection using a boosted cascade of simple features.
\newblock In \emph{Proc. CVPR}, 2001.

\bibitem[Weiss and Taskar(2010)]{weiss10}
D.~Weiss and B.~Taskar.
\newblock Structured prediction cascades.
\newblock In \emph{Proc. AISTATS}, 2010.

\bibitem[Weiss et~al.(2010)Weiss, Sapp, and Taskar]{weiss10ensemble}
D.~Weiss, B.~Sapp, and B.~Taskar.
\newblock Sidestepping intractable inference with structured ensemble cascades.
\newblock In \emph{Proc. NIPS}, 2010.

\end{thebibliography}

\end{document}